\documentclass[letterpaper,10pt]{ieeeconf}

\IEEEoverridecommandlockouts                              

\usepackage{amsmath,amssymb,amsfonts,amsthm,amscd,dsfont} 
\usepackage{mathtools}
\usepackage{algorithm} \usepackage{algpseudocode}
\usepackage[table,usenames,dvipsnames]{xcolor}
\usepackage{graphicx}
\usepackage{tabularx}
\usepackage{adjustbox}
\usepackage{subcaption}
\usepackage{verbatim}
\usepackage[noadjust]{cite} 

\newtheorem{proposition}{Proposition}

\theoremstyle{definition}
\newtheorem{assumption}{Assumption}
\newtheorem*{problem*}{Problem}

\title{\LARGE \bf
Distributed Gaussian Process Mapping for Robot Teams with Time-varying Communication
}

\author{James Di$^{1}$, Ehsan Zobeidi$^{2}$, Alec Koppel$^{3}$, Nikolay Atanasov$^{2}$
\thanks{$^{1}$ James Di is with ByteDance Inc, Mountain View, CA 94041 USA {\tt\small yubai.di@bytedance.com}.}%
\thanks{$^{2}$ Ehsan Zobeidi and Nikolay Atanasov are with Department of Electrical and Computer Engineering, University of California San Diego, La Jolla, CA 92093 USA {\tt\small \{ezobeidi, natanasov\}@ucsd.edu}.}%
\thanks{$^{3}$ Alec Koppel is with Supply Chain Optimization Technologies, Amazon, Bellevue, WA 98004 USA {\tt\small aekoppel@amazon.com}. Work completed while at the U.S. Army Research Laboratory in Adelphi, MD 20783}
}

\begin{document}
\maketitle
\thispagestyle{empty}
\pagestyle{empty}

\begin{abstract}
Multi-agent mapping is a fundamentally important capability for autonomous robot task coordination and execution in complex environments.
While successful algorithms have been proposed for mapping using individual platforms, cooperative online mapping for team of robots remains largely a challenge.
We focus on probabilistic variants of mapping due to its potential utility in down-stream tasks such as uncertainty-aware path-planning. A critical question to enabling this capability is how to process and aggregate incrementally observed local information among individual platforms, especially when their ability to communicate is intermittent. We put forth an Incremental Sparse Gaussian Process (GP) methodology for multi-robot mapping, where the regression is over a truncated signed-distance field (TSDF). Doing so permits each robot in the network to track a local estimate of a pseudo-point approximation GP posterior and perform weighted averaging of its parameters with those of its (possibly time-varying) set of neighbors. We establish conditions on the pseudo-point representation, as well as communications protocol, such that robots' local GPs converge to the one with globally aggregated information. We further provide experiments that corroborate our theoretical findings for probabilistic multi-robot mapping.
\end{abstract}

\section{Introduction}



\begin{figure}[t]
    \centering
    \includegraphics[width=\linewidth, trim={5cm, 3cm, 2cm, 4cm}, clip]{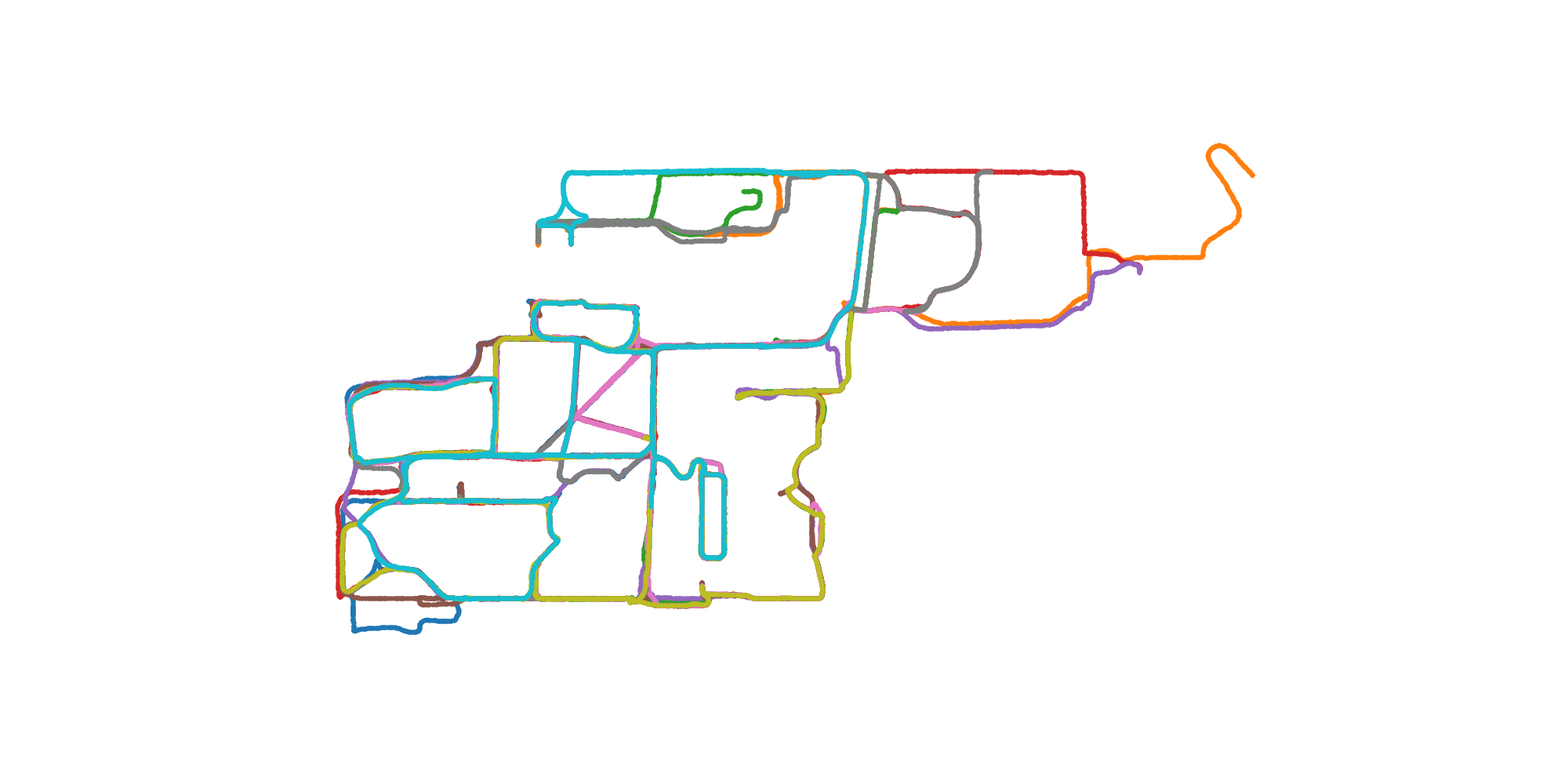}\\
    \includegraphics[width=\linewidth, trim={5cm, 4cm, 2cm, 4cm}, clip]{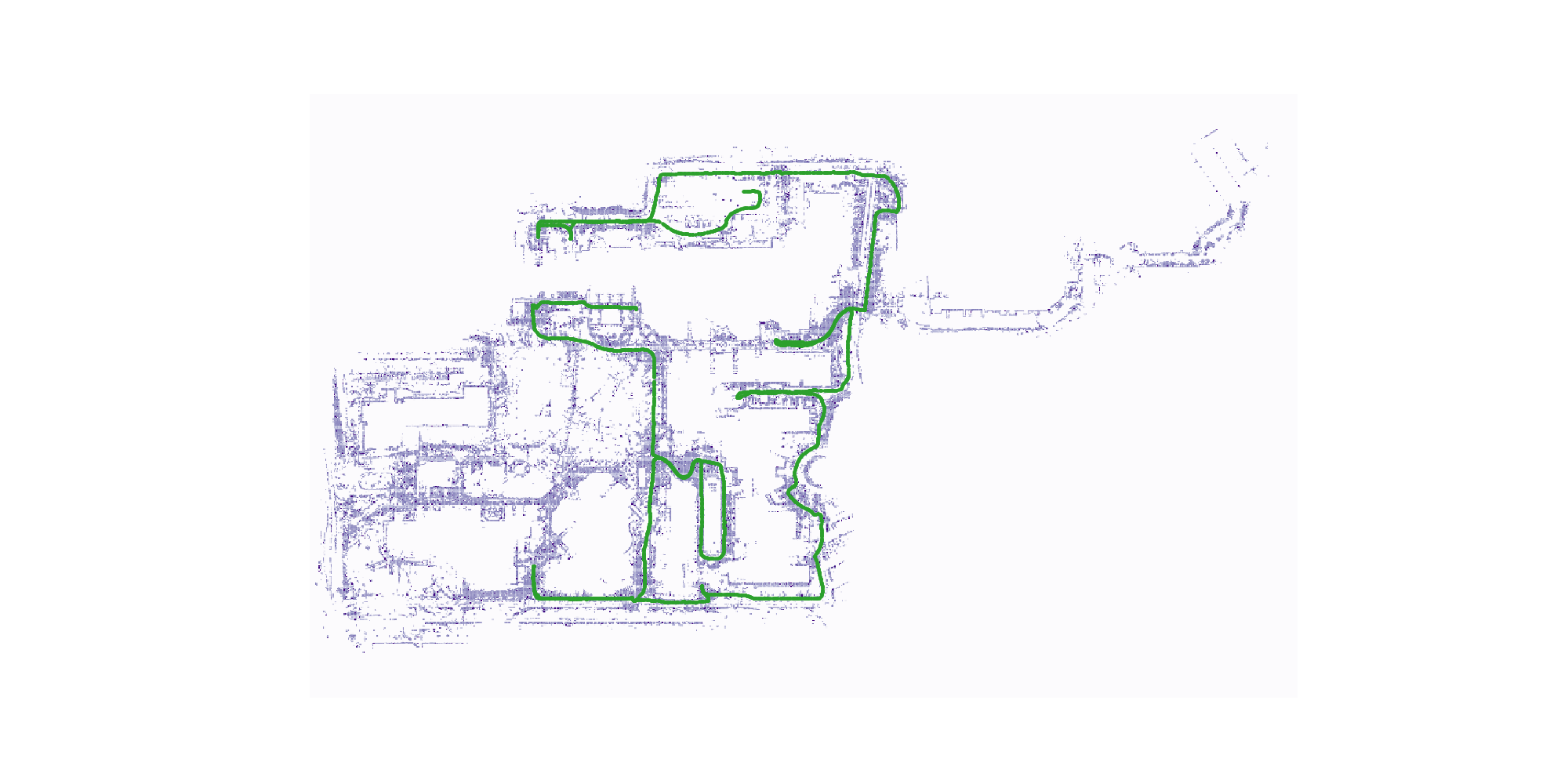}
    \caption{This work tackles probabilistic continuous-space mapping using point-cloud observations from a robot team with time-varying communication. The trajectories of $10$ sequences from the NCLT dataset \cite{nclt}, representing different robots, are shown at the top. The dataset provides 2D LiDAR scans and other sensor data collected in a square kilometer region of a university campus over the span of several months. In our setting, the robot communication is intermittent as robots move in and out of communication range. Each robot collaborates with its teammates to reconstruct a probabilistic truncated signed distance field (TSDF) map of the environment. The TSDF map and trajectory of robot 3 is shown at the bottom. Even though robot $3$ does not visit all parts of the environment, it obtains a complete TSDF map of the environment.}
    \label{fig:intro}
\end{figure}

Simultaneous Localization and Mapping (SLAM) refers to the ability of a robot to identify its location within an unknown environment while simultaneously constructing a map of its surroundings. SLAM is critically important to enable real-time robot operation, using only on-board sensing \cite{durrant2006simultaneous,thrun2007simultaneous}. We focus on a distributed mapping problem, where multiple robots acquire sensor data and seek to aggregate it to boost the statistical accuracy of their maps  \cite{doorslam,ccmslam,cunningham2010ddf,cunningham2013ddf}. This setting is important for reducing the amount of time to map an unknown environment at sufficiently high accuracy \cite{nieto2014coordination}. However, collecting and optimizing the data at a central location requires every robot to send all of its observations followed by a large-scale optimization, rendering the process excessively slow, expensive, and brittle in robotics applications with challenging communication infrastructure  \cite{lanbo2008prospects,alexis2019resilient}. Instead of using a centralized architecture, in this work we allow the robots to process their local observations incrementally and communicate in a \emph{distributed} manner. The challenge then becomes mixing the information received at each robot from its neighbors that are within communication range to collaboratively build the map, without knowing the robot network topology a priori.


In this work, we focus on using truncated signed distance field(TSDF) \cite{kinectfusion} as our environment representation. Compared to occupancy grid  \cite{occupancy-grid}, TSDF quantifies both occupancy and distance to the closest object surface, and is more attractive for down-stream tasks such as collision avoidance and reconstruction. However, updating the global map whenever new observations arrive is expensive. Hence to keep the computational complexity in check, inspired by map decomposition methods such as  Octomap \cite{hornung2013octomap} and voxel-hashing \cite{niessner2013real, voxblox}, we employ the data structure QuadTree(Octree in 3D) to reduce computation complexity, and only update maps of the relevant regions when new sensor observations are collected. 


Furthermore, we focus in the case that individual platforms seek to infer probabilistic information of their environmental map, as uncertainty quantification is important for down-stream tasks such as safe navigation and collision avoidance. Gaussian Process (GP) is a non-parametric model that provides a natural choice for tracking the posterior of the map while providing uncertainty information  \cite{rasmussen2003gaussian}. However, the training procedure of GP is cubic in the number of samples, and various methods  \cite{bauer2016understanding,koppel2021consistent} have been proposed to strike a balance between computational effort and statistical accuracy in GP inference. We develop an approach based off \emph{pseudo-point} approximations as in   \cite{snelson2006sparse}, which reduces the training complexity from $\mathcal{O}(n^3)$ to $\mathcal{O}(m^2n)$, where $n$ is the number of training samples, $m$ is the number of pseudo-points, and whose convergence has recently been characterized \cite{burt2019rates}. When $m \ll n$ this could lead to significant computational savings. 

We further develop a weighted averaging scheme for propagating distributions of TSDF estimated by individual GPs across the network, inspired by consensus protocols \cite{boyd2005gossip,nedic2009distributed}. While sub-optimal compared to approaches based on Lagrangian relaxation of consensus constraints, such as primal-dual method \cite{koppel2015saddle}, dual methods \cite{terelius2011decentralized}, and ADMM  \cite{boyd2011distributed}, the proposed approach is simple and efficient, making it suitable for distributed probabilistic inference. It has an elegance of simplicity that makes it natural to employ in distributed incremental GP inference setting. 

Thus, to achieve distributed, probablistic, online and efficient mapping with uncertainy information, we construct an algorithm based upon incremental sparse Gaussian Processes (GP) with  pseudo-input approximations, which are regressed over sequentially observed TSDF measurements taken by each platform.  Information mixing is executed through parametric representations of the GP parameters. A critical key point of departure of this work from that which preceded it \cite{zobeidi2021dense} is its ability to operate in a time-varying network without stale information. This capability is important in operational settings such as underwater \cite{lanbo2008prospects} and underground \cite{alexis2019resilient}, where communication is intermittent due to environmental effects. In addition, we demonstrate both in theory and practice on two public real-world datasets with multiple sequences that this method yields statistically consistent GP posteriors of the environmental map. In summary, the \emph{contributions} of this paper are to:
\begin{itemize}
    \item develop a distributed protocol for mixing incremental pseudo-points GP posterior of TSDF over a time-varying network,
    \item establish a convergence guarantee under suitable conditions on the pseudo-points, communications network, and input space,
    \item corroborate the proposed algorithm on two real-world LiDAR datasets, one of which is large-scale.
\end{itemize}
We demonstrate both theoretically and empirically that the proposed distributed mapping algorithm in a time-varying communication graph structure converges asymptotically to a centralized estimate, which relies on the information of all robots.

\section{Problem Statement}\label{sec:problem}
We consider the problem of mapping a $d$-dimensional environment ($d \in \{2,3\}$ in practice) with occupied space $\Omega \subset \mathbb{R}^d$ and free space $\mathcal{F} \subset \mathbb{R}^d$. We aim to estimate a truncated signed distance function (TSDF) \cite{curless1996volumetric,niessner2013real} as a continuous representation of the environment. The truncated signed distance from a point $x$ to the occupied space $\Omega$ is defined as:
%
\begin{equation}
    g(x) = \begin{cases}
\phantom{-}\min(d(x, \partial \Omega), h),& \text{if }  x \notin \Omega,\\
-\min(d(x, \partial \Omega), h),& \text{if } x \in \Omega,
\end{cases}
\end{equation}
where $h > 0$ is a pre-defined truncation value, $\partial \Omega$ is the boundary of $\Omega$, and:
\begin{equation}
d(x, \partial \Omega) = \inf_{y \in \partial \Omega} \|x-y\|_2.
\end{equation}
The TSDF $g(x)$ provides the (truncated) minimum distance from $x$ to the boundary of the occupied space and is negative if $x$ is within the occupied space.

We employ a team of $n$ robots to gather observations of the environment over a time horizon $0,\ldots,T$. The observation of robot $i$ at time $t$ is a point-cloud $Z_t^i \subset \mathbb{R}^d$ obtained, e.g., from a LiDAR scanner, depth camera, or another range sensor. We assume that the robot positions $p_t^i \in \mathbb{R}^d$ and orientations $R_t^i \in SO(d)$ are known for all $t$ and $i$ from an odometry algorithm, e.g., performing scan-matching \cite{csm} or pose graph optimization \cite{dellaert2017factor}. The world-frame coordinates of a point $z \in Z_t^i$ observed by the robot can be obtained via $p_{t}^{i} + R_{t}^{i} z$.

We assume the the team of robots are able to exchange the information over an undirected time-varying graph $G_t = (\mathcal{V}, \mathcal{E}_t)$ with nodes $\mathcal{V} = \{1,\ldots,n\}$, corresponding to the robots, and edges $\mathcal{E}_t \subseteq \mathcal{V} \times \mathcal{V}$. If two robots $i,j \in \mathcal{V}$ are able to communicate at time $t$, then an edge $(i,j) \in \mathcal{E}_t$ is present in the graph. The robots that robot $i$ can communicate with at time $t$ are called its neighbors and will be denoted by the set $\mathcal{N}_t^i = \{ j \in \mathcal{V} | (i,j) \in \mathcal{E}_t\}$. We aim to design a fully distributed TSDF mapping approach, in which the robots communicate only with their neighbors and place minimal restrictions on the communication structure. We consider time-varying networks, in which the graph $G_t$ may be instantaneously disconnected but the union of the graphs over a period of time $B$ is connected. This assumption is much weaker than requiring the robots to be in constant communication and is utilized for many results in multi-agent coordination and distributed optimization \cite{nedic2015time-varying,nedic2017achieving}.

\begin{assumption}\label{assumption:B}
The graph sequence $G_t=(\mathcal{V},\mathcal{E}_t)$ is \emph{uniformly connected}, i.e., there exists an integer $B>0$ (potentially unknown to the robots) such that the graph with node set $\mathcal{V}$ and edge set $\mathcal{E}_k^B = \bigcup_{t = kB}^{(k+1)B-1}\mathcal{E}_t$ is connected for all $k = 0, 1,\ldots$.
\end{assumption}

For each robot $i$, our goal is to incrementally infer a posterior distribution over the TSDF representation $g$ of the environment, conditioned on the sequential observations $Z^i_t$ of robot $i$ as well as the information received from its neighbors $\mathcal{N}^i_t$ for $t = 0, \ldots, T$. This amounts to an online distributed Bayesian inference problem over a time-varying network. Our approach to this problem is described in the following section.

\section{Technical Approach}\label{sec:algorithm}

We organize our technical approaches into the following sections. Section \ref{single_agent} discusses the TSDF estimation framework for a single agent, leveraging sparse pseudo-point Gaussian Process. Section \ref{multi_agent} presents the distributed update protocol over time-varying graph of robots, with theoretical proposition and proof over the convergence of the algorithm.

\subsection{Regressing TSDF via pseudo-point Gaussian Processes}
\label{single_agent}

To obtain TSDF $g(x)$ over the entire environment for a single agent, we leverage pseudo-point approximations of Gaussian Processes to provide a way to infer the distribution in a parametrically efficient manner. Before doing so, we review the key steps of Gaussian Process and sparse pseudo-points GP in the following sections. 

{\bf \noindent Gaussian Processes.} A Gaussian Process (GP) $f(x) \sim \mathcal{GP}(\mu(x),  k(x, x'))$ is a stochastic process such that any finite collection of its realizations is jointly Gaussian. It is parameterized by its mean function $\mu(x) $ and a covariance kernel $k(x, x')$, where $\mu(x), k(x, x)'$ are initialized with prior functions $\mu_{0}(x), k_{0}(x, x')$. We employ this nonparametric model to hypothesize that the robot's observations are corrupted by zero-mean Gaussian noise: $y_{i} = f(x_{i}) + \epsilon_{i}$, where the noise prior satisfies $\epsilon _{i} \sim N(0, \sigma_{i}^2)$.  As the robot sensors collect observations $\mathcal{X} = \{(x_i, y_i) \}$, the posterior of the GP is Gaussian and can be written as $f(x) \vert \mathcal{X} \sim \mathcal{GP}(\mu(x), k(x, x'))$ with associated conditional mean and covariance functions:
\begin{equation}
\label{gp}
\begin{aligned}
    \mu(x) & = \mu_{0}(x)\! +\! k_{0}(x, \mathcal{X})(k_{0}(\mathcal{X}, \mathcal{X}) \!+\! \sigma^{2}I)^{-1}(y\! -\! \mu_0(\mathcal{X})) \\
    k(x, x') & = k_{0}(x, x') \!-\! k_{0}(x, \mathcal{X})(k_{0}(\!\mathcal{X}, \mathcal{X}) \!+\! \sigma^{2}I)^{-1}k_{0}(\mathcal{X}, x') \\
    & + \sigma^{2}
\end{aligned}
\end{equation}

{\bf \noindent Complexity Reduction.}
The computational complexity of the training procedure is $O(n^{3})$ with sample size $n=|\mathcal{X}|$, due to the inversion of the kernel matrix $k_{0}(\!\mathcal{X}, \mathcal{X})$ in \eqref{gp}. We adopt an approximation based on a set of pseudo-points  $\mathcal{P} \subset \mathcal{X}$, where $\vert \mathcal{P} \vert = m \ll n$, as in \cite{snelson2006sparse}. The key is that we model the target $y$ to come from a latent function $f$, and we parametrize the model with the pseudo-points $\mathcal{P}$ and their pseudo targets $\bar{f}$. The distribution of targets $y$ conditioned on the locations $X$, $\mathcal{P}$ and $\bar{f}$, i.e. $p(y \vert X, \mathcal{P}, \bar{f})$ is Gaussian with parameters:

\begin{equation}
    \begin{aligned}
        \mu(x) & = \mu_{0}(x)\! +\! k_{0}(x, \mathcal{P})(k_{0}(\mathcal{P}, \mathcal{P}) + \sigma^2 I)^{-1}(y\! -\! \mu_0(\mathcal{P})) \\
        k(x, x') & = k_{0}(x, x') \!-\! k_{0}(x, \mathcal{P})(k_{0}(\mathcal{P}, \mathcal{P}) + \sigma^2 I)^{-1}k_{0}(\mathcal{P}, x') \\
        & + \sigma^{2}
    \end{aligned}
    \label{eq:likelihood}
\end{equation}

Since we can assume the pseudo-targets come from the same distribution as the dataset $\mathcal{X}$, we place the same prior on $\bar{f} \sim N(\mu_0, k_0)$, and after using Bayes Rules on \eqref{eq:likelihood} and the prior we can write $\bar{f} \vert \mathcal{X}, \mathcal{P}$ as:
\begin{equation}
    \begin{aligned}
        \mu(\mathcal{P}) & = \mu_{0}(\mathcal{P}) \!+ k_0(\mathcal{P}, \mathcal{P})(k_{0}(\mathcal{P}, \mathcal{P}) \!+ \Gamma)^{-1} \gamma \\
        \Sigma(\mathcal{P}) & = k_0(\mathcal{P}, \mathcal{P})(k_{0}(\mathcal{P}, \mathcal{P}) \! + \Gamma)^{-1}k_{0}(\mathcal{P}, \mathcal{P})
    \end{aligned}
    \label{eq:pseudo_likelihood}
\end{equation}
with weighting factors $\Gamma = k_{0}(\mathcal{P}, \mathcal{X})(\Lambda + \sigma^2 I)^{-1}k_{0}(\mathcal{X}, \mathcal{P})$, $\Lambda = k_0(\mathcal{X}, \mathcal{X}) - k_{0}(\mathcal{X}, \mathcal{P})k_{0}(\mathcal{P}, \mathcal{P})^{-1} k_{0}(\mathcal{P}, \mathcal{X})$, and $\gamma = k_{0}(\mathcal{P}, \mathcal{X})(\Lambda + \sigma^{2}I)^{-1}(y - \mu_{0}(X))$.

Using the definition of info matrix $\Omega(\mathcal{P}) = \Sigma(\mathcal{P})^{-1}$ and information mean $\omega (\mathcal{P}) = \Omega \mu(\mathcal{P})$, equivalently to \eqref{eq:pseudo_likelihood} the info mean and info matrix of $\mathcal{P}$ can be written as:
\begin{equation}
\begin{aligned}
        \omega(\mathcal{P}) & = \Omega \mu_{0}(\mathcal{P}) + k_{0}(\mathcal{P}, \mathcal{P})^{-1} \gamma \\
        \Omega(\mathcal{P}) & = k_{0}(\mathcal{P}, \mathcal{P})^{-1} (k_{0}(\mathcal{P}, \mathcal{P}) + \Gamma)k_{0}(\mathcal{P}, \mathcal{P})^{-1}
\end{aligned}
\label{eq:spgp_info_mean}
\end{equation}

Finally, integrating out $\bar{f} \vert \mathcal{X}, \mathcal{P}$ as in  \eqref{eq:spgp_info_mean} and $f(x) \vert X, \mathcal{P}, \bar{f}$ as in \eqref{eq:likelihood}, the targets are distributed as Gaussian with parameters:

\begin{equation}
\begin{aligned}
\mu(x) & = \mu_{0}(x) \! + k_{0}(x, \mathcal{P})k_{0}(\mathcal{P}, \mathcal{P})^{-1}(\Omega^{-1}\omega - \mu_{0}(\mathcal{P})) \\
k(x, x') & = k_{0}(x, \mathcal{P})k_{0}(\mathcal{P}, \mathcal{P})^{-1} \Omega^{-1} k_{0}(\mathcal{P},\mathcal{P})^{-1}k_{0}(\mathcal{P}, x') \\
& \!+ k_{0}(x, x') \!  - k_{0}(x, \mathcal{P})k_{0}(\mathcal{P}, \mathcal{P})^{-1}k_{0}(\mathcal{P}, x')
\end{aligned}
\label{eq:sparse_gp}
\end{equation}

This is a key complexity reduction of GP posterior computations. However, it does not address the fact that in incremental settings, the robots may encounter repeated pseudo-points, which is likely to happen as one robot may visit places already covered another, and can cause memory issues. Hence we incorporate the pseudo-point aggregation technique in \cite{zobeidi2020dense}, which is proven to have the same distribution as the pseudo-point GP mean and covariance in \eqref{eq:sparse_gp}. The GP using the aggregated statistics have mean and covariance:
\begin{equation}\label{eq:compressed_sp_gp}
\begin{aligned}
\mu(x) & = \mu_{0}(x) + k_{0}(x, \mathcal{P})Z(\zeta - \mu_0(\mathcal{P})) \\
k(x, x') & = k_{0}(x, x') - k_{0}(x, \mathcal{P})Zk_{0}(\mathcal{P}, x')
\end{aligned}
\end{equation}
where $Z^{-1} \coloneqq k_{0}(\mathcal{P}, \mathcal{P}) + \sigma^{2} \text{diag}(m)^{-1}$. $\mathcal{P} \subset \mathcal{X}$ are the locations of the pseudo-points, and with abuse of notation $\zeta(x)$ is the arithmetic average of the observations for $x \in \mathcal{P}$, and $m(x)$ is the occurrence count of observations for $x \in \mathcal{P}$.

{\bf \noindent Tree Data Structures.} To keep the computational effort under control, for each agent $i$ we use one QuadTree $\textit{Tr}^{i}_t$ (Octree in 3D) to keep track of the pseudo points at time $t$. Let $L^{i, k}_{t}$ denote the $k$th QuadTree $\textit{leaf}$ node for tree $\textit{Tr}^{i}_t$, which may be the root when the tree has not been split, or a child of $\textit{Tr}^{i}_t$. When new pseudo-points are inserted, if the number of pseudo-points in $L^{i,k}_t$ exceeds a threshold \textit{maxLeafSize}, the node $L^{i,k}_t$ is split recursively until all leaves of $L^{i,k}_t$ have less than \textit{maxLeafSize} pseudo-points.

\begin{figure}[t!]
    \centering
     \includegraphics[trim={1cm, 1cm, 2cm, 20
     cm}, clip,width=0.45\textwidth]{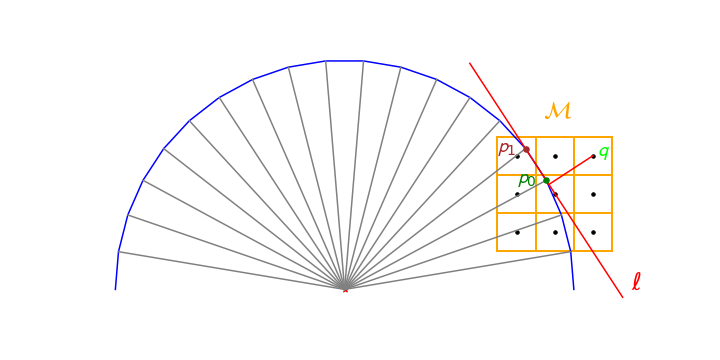}
     \caption{Approximation of  pseudo-points' TSDF in 2D, which is used as training samples (along with pseudo-points locations) for GP. The gray lines are LiDAR rays from the scanner. The blue line represents a object's surface boundary. The green and brown points $p_0, p_1$ are two adjacent laser rays' hit point with the boundary. Pseudo-points (black) are placed over a local frame $\mathcal{M}$ of the point $p_0$. The TSDF of a pseudo-point $q \in \mathcal{P}$ is approximated with the distance to the line $l$ passing the two adjacent laser hit points.}
     \label{fig:tsdf_approx}
     \vspace{-10pt}
\end{figure}

{\bf \noindent Pseudo-point Placement.} To place the pseudo-points in the environment, we adopt a grid-approach, and place a local frame $\mathcal{M}$ over the laser endpoints, as proposed by \cite{zobeidi2020dense}. Since the laser beams are hit where there are obstacles, we can assume the endpoints of the laser beams have $g(x) = 0$.  As illustrated in Fig.\ref{fig:tsdf_approx}, a 3x3 grid is placed over the laser endpoint, and the pseudo-points are selected from each grid point within the frame $\mathcal{M}$ such that pseudo-points of both positive and negative TSDF values are selected. Thus with the observation $Z^i_t$ we locate the set of pseudo points $\tilde{\mathcal{P}}^i_t$.

{\bf \noindent TSDF Approximation in 2D.} For each pseudo-point $q \in \tilde{\mathcal{P}}^i_t$, to compute the TSDF corresponding to the recent observation $Z^i_t$, suppose that  $p_{1}, p_{2} \in X$ are two adjacent laser hit points, and  $\vec{p_1}, \vec{p_2}$ denote the corresponding vectors starting from the origin $O$. We assume that $p_{1}, p_{2}$ belong to the same boundary surface $\partial \Omega_{j}$, and approximate the distance field of a pseudo-point $q \in \mathcal{M}(p_{1})$ as the distance $d(q, \overline{p_1 p_2})$, where $\overline{p_{1} p_{2}}$ denotes the line passing through $p_1 p_2$ in $2D$, or  the plane containing $p_{1}, p_{2}$ in 3D, with the third adjacent laser hit point. In 2D, $d(q, \overline{p_1 p_2})$ can be calculated using the following equation:
\begin{equation}
\label{eq:tsdf_transformation}
    d(q, \overline{p_1 p_2}) = \frac{\lvert (x_2 - x_1)(y_1 - y_0) - (x_1 - x_0)(y_2 - y_1) \rvert}{\sqrt{(x_2 - x_1)^2 + (y_2 - y_1)^2}}
\end{equation}
where $q = (x_0, y_0)$, $p_1 = (x_1, y_1)$, $p_2 = (x_2, y_2)$. We obtain $g(q) = \min(d(q, \overline{p_1 p_2}), h)$ if $q$ is in the same halfspace of robot's position $x_{t}$, or $-\min(d(q, \overline{p_1 p_2}), h)$ otherwise. 

Then we have a set of pseudo points $\Tilde{\mathcal{P}}^i_t$ with their estimated TSDF values corresponding to observation $Z_t^i$. The transformed data are samples of $g$ function that are approximated with Gaussian Process, with mean $\mu(x)$ and covariance $k(x, x)$ defined in \eqref{eq:compressed_sp_gp}.

\subsection{Distributed Updates in Time-Varying Network}

\label{multi_agent}
To make the updates for the GP approximation scoped in the previous subsection executable in a distributed setting, we consider a setting where each robot acquires a local observation stream. Then, it must incorporate message passing with others in order to propagate its local information across the network and achieve a degree of coordination. In the probabilistic mapping setting, we should not expect each robot's local map to coincide, since it has traversed a different path. Cooperation is then defined from the perspective of whether individuals estimates approach that of centralized meta-agent that has access to all robots' local information. To formalize whether agents' parameters approach that of the centralized agent, conditions are required on the network and a ground truth, so as to quantify convergence. In this work we assume \textit{B-connected graph sequence} defined by Assumption~\ref{assumption:B}, and we will prove it leads to convergence.

Let $W_t$ be the weight matrix of the network at time $t$. We assume that $[W_t]_{ii}>0$ for all $i \in V$. At each time step, the edge set $\mathcal{E}_t$ is determined via robots' distance from each other. When two robots are closer than a distance threshold, there is an edge between them; otherwise, they are not connected. Given the communication graph $G_t$, a practical weight matrix that is widely used in distributed literature is Metropolis weights \cite{boyd2005gossip} \cite{nedic2017achieving}, and is defined as:
\begin{equation}
\label{weight_matrix}
[W_t]_{i,j} = \begin{cases}
      \frac{1}{(1+\max(\deg^i_t,\deg^j_t))}, & \text{if } $(i,j)$ \in \mathcal{E}_t,\\
      0, & \text{if } $(i,j)$ \notin \mathcal{E}_t,\\
      1-\sum_{j \in Ne^i_t}[W_t]_{i,j}, & \text{if } $i = j$, 
      \end{cases}
\end{equation}
where $deg^i_t$ is the number of neighbors for robot $i$ at $t$. As established in \cite{nedic2017achieving}, one important property of Metropolis weights, is that Assumption~\ref{assumption:B} leads to $\Pi_{t=\tau}^\infty W^t  \to \frac{\bold{1} \bold{1}^\top}{n}$. More precisely, it has uniform stationary distribution, for any starting time $\tau$.

%
%
 
%
To evaluate the convergence of our distributed algorithm we obtain the GP estimation of function $g$ through a centralized agent. We define centralized agent as an imaginary robot that receives all the observations at each time $t$ from all robots. We define the central agent's update of its GP parameters as the following based on the stationary distribution:
%
\begin{align}\label{eq:central}
  \mathcal{P}_{t+1}^{ctr} &= \cup_{i=1}^n \tilde{\mathcal{P}}_{t+1}^{i} \cup \mathcal{P}_t^{ctr},\notag\\ 
	m_{t+1}^{ctr}(\bold p) &= m_{t}^{ctr}(\bold p) + \sum_{i=1}^n \frac{\tilde{m}_{t+1}^{i}(\bold p)}{n},\\
  \zeta_{t+1}^{ctr}(\bold p) &= \frac{m_{t}^{ctr}(\bold p)\zeta_t^{ctr}(\bold p)+\frac{1}{n}\sum_{i=1}^n \tilde{m}_{t+1}^{i}(\bold p)\tilde{\zeta}_{t+1}^i(\bold p)}{m_{t+1}^{ctr}(\bold p)},\notag
\end{align}
%

Next we introduce our method to update the robotic network. For each robot $i$ at time step $t$, we transform the observation $\{Z^{i}_{t} \}$ into a set of pseudo-point statistics, as explained in Sec.\ref{single_agent}. Using \eqref{eq:tsdf_transformation} for each new pseudo point $\bold{p} \in \tilde{\mathcal{P}}^i_{t}$ we can summarize the transformed data by averaging the TSDF values into $\tilde{\zeta}^i_t (\bold{p})$ and corresponding counting number of observation in $\tilde{m}^i_t (\bold{p})$. We also need a list that contains the robots that received this mini-batch to make sure each robot, receives each mini-batch exactly once. 
By combining all parameters for each observation $\tilde{\Theta}_t^i := \{ \tilde{\mathcal{P}}_t^i, \tilde{m}_t^i(\tilde{\mathcal{P}}_t^i), \tilde{\zeta}_t^i(\tilde{\mathcal{P}}_t^i), \ell_t^i \}$ define a mini-batch of observations for robot $i$, where $\ell^{i}_t$ is the list of robots that has received the mini-batch up to time $t$. In addition, $\tilde{B}^{i}_t$ defines the set of mini-batches robot $i$ receives at t, and $B^{i}_t$ defines all the batches $i$ retains. The distributed protocol for updating the GP parameters can be summarized as:

%
\begin{equation}
\label{eq:echoless_gp_update}
\begin{aligned}
\mathcal{\tilde{B}}_{t+1}^i &= \bigcup_{\tilde{\Theta}_\tau^j \in\mathcal{B}_t^r, r\in \mathcal{N}^{i}_t, i \not \in \ell_{\tau}^{j}} \tilde{\Theta}_\tau^j \cup \tilde{\Theta}_{t+1}^i\\
\mathcal{B}_{t+1}^i &= \mathcal{B}_{t}^i \cup \mathcal{\tilde{B}}_{t+1}^i\\
\ell_\tau^j &= \ell_\tau^j \cup \{i\} \text{ for all } \tilde{\Theta}_\tau^j \in\mathcal{\tilde{B}}_{t+1}^i\\
\mathcal{P}_{t+1}^i &= \bigcup_{\tilde{\Theta}_\tau^j \in \mathcal{\tilde{B}}_{t+1}^i} \mathcal{P}_{\tau}^j \cup\mathcal{P}_t^i\\
m_{t+1}^{i}(\bold{p}) &= m^i_{t}(\bold{p}) + \sum_{\tilde{\Theta}_\tau^j\in \mathcal{\tilde{B}}_{t+1}^i} \frac{\tilde{m}^j_\tau(\bold{p})}{n}\\
\zeta^i_{t+1}(\bold{p}) &= \frac{m^i_{t}(\bold{p}) \zeta^i_{t}(\bold{p}) + \frac{1}{n}\sum_{\tilde{\Theta}_\tau^j \in \mathcal{\tilde{B}}_{t+1}^i} \tilde{m}^j_\tau(\bold{p}) \tilde{\zeta}^j_{\tau}(\bold{p})}{m^i_{t+1}(\bold{p})}
\end{aligned}
\end{equation}
where a mini-batch is expired from a robot $i$ when its list $\ell_{t}^{i}$ is full. In the mentioned protocol, for convenience of representation we considered the list of robots for each mini-batch. In practice the copies of a mini-batch do not have to communicate. Instead, due to Proposition~\ref{prop:1} it is sufficient that when a robot received a mini-batch $\tilde{\Theta}^j_\tau$ at time $t$, keeps it for $(\lceil\frac{\tau}{B}\rceil + (n - 1))B-t$ time steps, and then remove it. By then all robots have received it. More precisely, the first two lines of \eqref{eq:echoless_gp_update} are processed on the network level, and the rest in practice are processed on the QuadTree node level.
\begin{proposition}\label{prop:1}
Let $Z^i_t$ be the data observed by robot $i$ at time $t$, associated with pseudo-points $\tilde{\mathcal{P}}_t^i \subset \mathcal{P}$, number of observations $\tilde{m}_t^i(\bold{p})$ and average observation $\tilde{\zeta}_t^i(\bold{p})$ for $\bold{p} \in \mathcal{P}$. If the data streaming stops at some time $T < \infty$, then at time $t =(\lceil\frac{T}{B}\rceil + (n - 1))B$, the distributions $\mathcal{G}\mathcal{P}(\mu_t^i(\bold{x}), k_t^i(\bold{x},\bold{x}'))$ maintained by each robot $i$, specified according to \eqref{eq:compressed_sp_gp} with parameters in \eqref{eq:echoless_gp_update} are exactly equal to the distribution $\mathcal{G}\mathcal{P}(\mu_t^{ctr}(\bold{x}), k_t^{ctr}(\bold{x},\bold{x}'))$ of the centralized estimator with parameters in \eqref{eq:central}, i.e., $\mu_t^i(\bold{x}) = \mu^{ctr}_t(\bold{x})$ and $k_t^i(\bold{x},\bold{x}') = k^{ctr}_t(\bold{x},\bold{x}')$ almost surely for all $i \in \mathcal{V}$, $\bold{x},\bold{x}'$.
\end{proposition}


\begin{proof}
With regard to Eq \eqref{eq:compressed_sp_gp}, it is sufficient to show that at $t = (\lceil\frac{T}{B}\rceil + (n - 1))B$, $m_t^i(\bold{p}) = m^{ctr}_t(\bold{p})$ and $\zeta_t^i(\bold{p}) = \zeta^{ctr}_t(\bold{p})$ for all $i \in V$, $\bold{p} \in \mathcal{P}$. We express $m_t^i(\bold{p})$ and $\zeta_t^i(\bold{p})$ in terms of $\tilde{m}_\tau^j(\bold{p})$ and $\tilde{\zeta}_\tau^j(\bold{p})$ for arbitrary $\bold{p} \in \mathcal{P}$ and $\tau \leq t$. The key is to realize whether mini-batch $\tilde{\Theta}_\tau^j$ is received by robot $i$. Since the mini-batch exchanges are happening based on the communication graph structure, the elements of $\Pi^{t-1}_{s=\tau}W_s$ determine which robots have received a mini-batch released at time $\tau$ by time $t$. Precisely, if $[\Pi^{t-1}_{s=\tau}W_s]_{ij} > 0$, then robot $i$ has received mini-batch $\tilde{\Theta}_\tau^j$ by time $t$ and otherwise, if $[\Pi^{t-1}_{s=\tau}W_s]_{ij} = 0$, it has not received it. Let $sign(x)$ denote the sign of a scalar $x$ with $sign(0) = 0$. Expanding \eqref{eq:echoless_gp_update} recursively leads to:
\begin{align}
\label{eq:k-expansion}
m_t^i(\bold{p}) &= \sum_{\tau=0}^{t}\sum_{i=1}^n sign([\Pi^{t-1}_{s=\tau}W_s]_{ij}) \frac{\tilde{m}_\tau^j(\bold{p})}{n}\\
\zeta^i_t(\bold{p}) &= \frac{1}{m^i_{t}(\bold{p})}\sum_{\tau=0}^{t}\sum_{j=1}^n sign([\Pi^{t-1}_{s=\tau}W_s]_{ij})  \frac{\tilde{m}_{\tau}^j(\bold{p})}{n}\tilde{\zeta}^j_{\tau}(\bold{p})\notag
\end{align}
Since the data collection stops at some finite time $T$, $\tilde{m}_\tau^i(\bold{p}) = \tilde{\zeta}_\tau^i(\bold{p}) = 0$ for all $\tau > T$, $i \in V$. Assumption~\ref{assumption:B} leads to the fact that $sign([\Pi^{(n-1)B+\tau-1}_{s=\tau}W_s]_{ij})>0$ for all $i,j \in V$, and it is not necessary true for $sign([\Pi^{(n-1)B+\tau-2}_{s=\tau}W_s]_{ij})$. In this regard, at time step $t$, let $A_t$ be set of all robots who received the mini-batch produced by robot $i$ at time step $\tau$, and $\bar{A}_t$ be the set of rest of robots who did not received it. At first $A_\tau=\{i\}$ and rest of robots are in $\bar{A}_\tau$ including $j$. Note that Assumption~\ref{assumption:B} means that in each $B$ step there is a time step that there is an edge between $A_t$ and $\bar{A}_t$, otherwise there will be $B$ time steps such that the graph will be unconnected. At such a point, we can consider the end of this edge in $\bar{A}_t$, remove it and add it to $A_t$. Since at the beginning there is $n-1$ robots in $\bar{A}_\tau$ at some point in $(n-1)B$ steps $j$ will be added to $A_t$. Comparing \eqref{eq:k-expansion} and \eqref{eq:central}, concludes the equality $\mu_t^i(\bold{x}) = \mu^{ctr}_t(\bold{x})$ and $k_t^i(\bold{x},\bold{x'}) = k^{ctr}_t(\bold{x},\bold{x'})$ at $t =(\lceil\frac{T}{B}\rceil + (n - 1))B$.
\end{proof}

\begin{table}[]
    \centering
    \caption{List of Notations}
    \label{table:list_of_symbols}
    \def\arraystretch{1.15}
    \begin{tabular}{|c|c|}
        \hline
        Symbols & Annotations \\
        \hline
        $\tilde{\Theta}^i_t$ & Mini-batch containing the transformed 
        local observation \\
        & robot $i$ observes at time $t$\\
        $\tilde{\Delta}^{i,k}_t $&  The mini-batches robot $i$ sends to $k$ at $t$\\
        $\tilde{\mathcal{B}}^i_t $&  The mini-batches robot $i$ receives at $t$\\
        $\mathcal{B}^i_t $&  The mini-batches robot $i$ retains at $t$\\
        $\mathcal{P}$ & Locations of the pseudo-points  \\
        $\zeta(p)$ & (Averaged) observation for a pseudo-point $p \in P$ \\
        
        $m(p)$ & Count of observation\\
        $\bold{\Tilde{\mathcal{P}}}^{i}_t$ & Locally observed pseudo point for robot $i$ at $t$\\
        
        $\bold{\mathcal{P}}^{i}_t$ &  All Pseudo-points robot $i$ retains at $t$ \\
        
        $\bold{\Bar{\mathcal{P}}}^{i}_t$ & Pseudo-points robot $i$ uses to update $\mathcal{GP}$ at $t$ \\
        \hline
    \end{tabular}
\end{table}

\begin{figure}[h!]
    \centering
    \includegraphics[width=1.0 \linewidth]{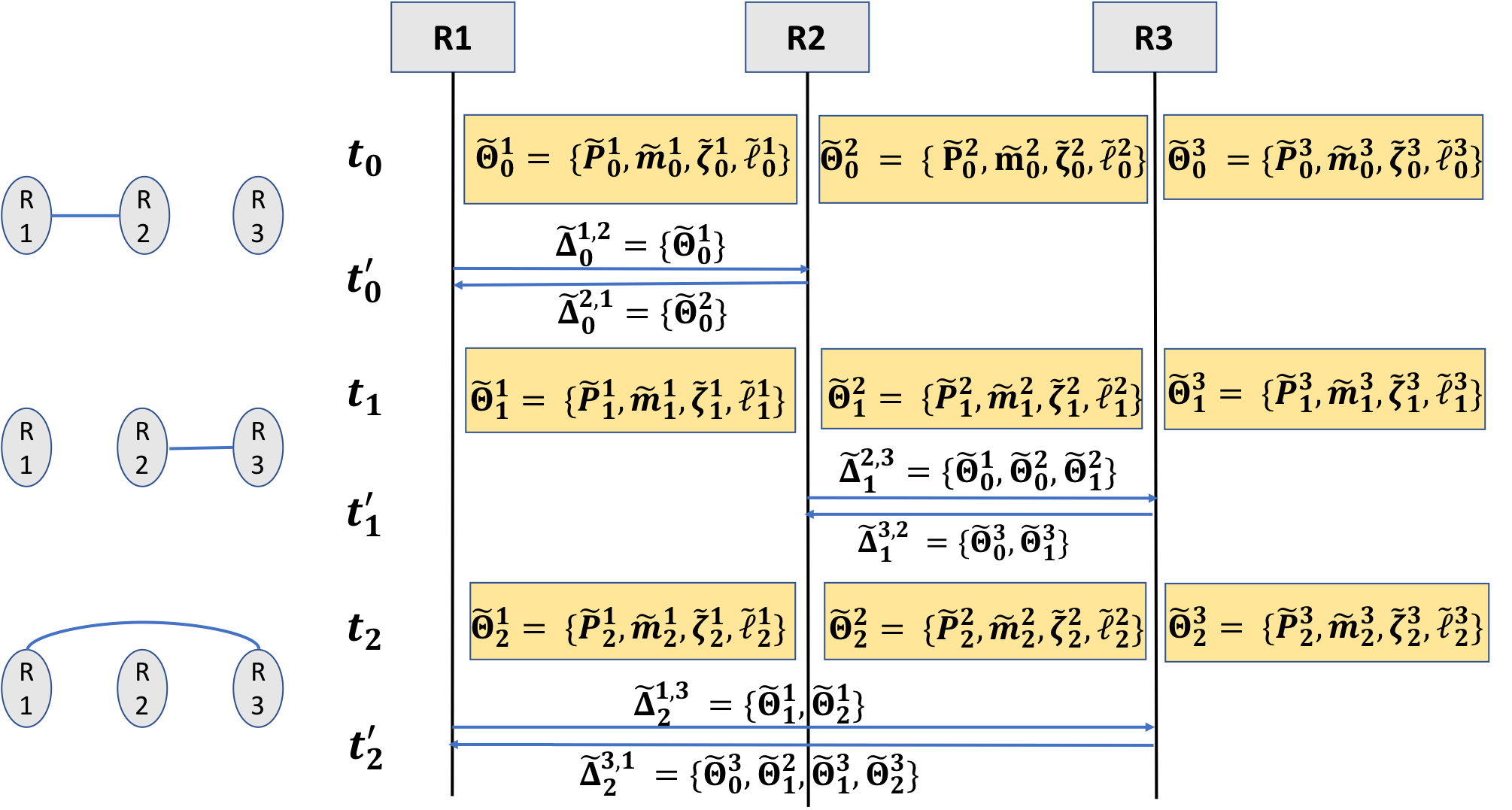}
    \caption{Diagram for message passing between agents, when communication graph $G_t$ is time-varying and $\lvert V \rvert = 3$. The yellow box denotes the event when new batch of pseudo-point and TSDF is observed, and a new message is constructed (Line (1-2) of Algorithm \ref{alg:distributed_mapping}). The left-most column represents $G_t$. To de-couple sensing and communication, at each time step $t$ new observation are collected, and at $t'$ messages are sent. (Line 4-5 of Algorithm \ref{alg:distributed_mapping})}
    \label{fig:seq_diagram}
\end{figure}

Synthesizing section III.A and the distributed protocol as proposed in \eqref{eq:echoless_gp_update}, we put forth the following algorithm for distributed update of pseudo-points statistics of the TSDF. The notations for the list of symbols can be found in the Table \ref{table:list_of_symbols}, and an illustration for the message passing between the agents can be found in Fig \ref{fig:seq_diagram}. The GP parameters $\mathcal{GP}^{i}_{t+1}$ for the corresponding QuadTree Node $T^{i}_{t+1}$ is updated according to \eqref{eq:echoless_gp_update} (Line 10 of Algorithm \ref{alg:distributed_mapping}). Finally, after merging the pseudo-points statistics from the neighbors, if the QuadTree Node $T^{i}_{t+1}$ has number of pseudo-points larger than \textit{maxLeafSize}, the node including  its pseudo-point statistics is split recursively (Line 12).

\begin{algorithm}
\caption{Distributed Mapping for robot $i$}\label{alg:cap}
\begin{algorithmic}[1]
\Require{$\text{new local observations } z_{t+1}^{i}, \text{robot position } x_{t+1}^{i}$
$\text{Octree }  T^{i}_{t}$,  $\mathcal{GP}^{i}_{t}$, $\text{Neighbors } Ne^{i}_{t+1}$}

\Ensure{$\text{Octree } T^{i}_{t+1}$, $\mathcal{GP}^i_{t+1}$}
\State $\Tilde{P}_{t+1}^{i}, \Tilde{\zeta}_{t+1}^{i}, \Tilde{m}_{t+1}^{i} \gets \text{computePseudoPoints}(x_{t+1}^{i}, z_{t+1}^{i})$
\State $\Tilde{\Theta}_{t+1}^{i} \gets \{ \tilde{\mathcal{P}}_{t+1}^i, \tilde{m}_{t+1}^i(\tilde{\mathcal{P}}_{t+1}^i), \tilde{\zeta}_{t+1}^i(\tilde{\mathcal{P}}_{t+1}^i), \ell_{t+1}^i \}$
\For {$j \in Ne^{i}_{t+1}$}
    \State ${\Tilde{\Delta}}_{t+1}^{i,k} = \Tilde{\Theta}^{i}_{t+1} \underset{\Tilde{\Theta}^{j}_{\tau} \in \mathcal{B}^{i}_t, j \not \in l^{k}_{t} }{\bigcup} \Tilde{\Theta}^{k}_{\tau}$
    \State \text{sendToNeighbor}$(\Tilde{\Delta}^{i,k}_{t+1})$ 
\EndFor
\State $\Bar{P}_{t+1}^{i}  \gets \tilde{P}^{i}_{t+1} \underset{j \in Ne^{i}_{t+1}, \Tilde{\Theta}^{j}_{\tau} \in \Tilde{\Delta}^{j, i}_{t}}{\bigcup} \Tilde{P}^{j}_{\tau}$
\State $T^{i}_{t+1} = T^{i}_{t} $
\For{$p_{t+1}^{i} \in \bar{P}_{t+1}^{i}$}
    \State
    \text{$T^{i}_{t+1}$.insert}$(p_{t+1}^{i}, \Bar{\zeta}_{t+1}^{i}(p_{t+1}), \Bar{m}_{t+1}^{i}(p_{t+1})$)
\EndFor


\State $T^{i}_{t+1} \gets \text{recursiveSplitNode}(T^{i}_{t+1})$ 
\end{algorithmic}

\label{alg:distributed_mapping}
\end{algorithm}

\section{Experimental Results}

We evaluate our algorithm on two public LiDAR datasets. The first dataset is from the Robotics Data Set
Repository (Radish) \cite{Radish}, which contains LiDAR sequences and odometry recorded from different environments. For each environment, we divide the sequence into equal lengths to simulate data obtained from a robot team. The second dataset is the North Campus Long-Term (NCLT) dataset \cite{nclt}, which contains multiple long sessions for a single robot navigating on University of Michigan's campus. The NCLT dataset includes challenging situations such as moving obstacles and different weather and lighting conditions. We show that using the time-varying distributed mapping algorithm (Algorithm \ref{alg:distributed_mapping}) the estimation of the \textit{TSDF} by the distributed team of robots  converge to that of the centralized agent, both qualitatively and quantitatively.


To evaluate the accuracy of the TSDF mapping, we evaluate the predictions of each individual agent with respect to the prediction of the centralized agent according to the following metric:
\begin{equation}
    \textit{RMSE}^{i} = \sqrt{\frac{\sum_{p \in X, \hat{z}(p) \in (-h, h)}(\hat{y}^{i}(p) - \hat{z}(p))^2} {\sum_{p \in X} \mathds{1}(\hat{z}(p) \in (-h, h))}}
\label{eq:rmse}
\end{equation}
where $\hat{z}(p)$ is the prediction of the centralized agent, $h$ is the truncated distance, $\hat{y}^{i}(p)$ is the prediction of the $i$th agent, and $X$ denotes the entire environment. We demonstrate the results first on the distributed agents without any communication, i.e. each agent maps the own area it has visited, and then on the team of agents using the distributed update algorithm. To simulate the neighborhood $Ne^{i}_{t}$ of the robot $i$, we compute the distance between all pairs of robots at all $t$, and connect two robots if the distance is lower than a threshold $r$. 

\subsection{Radish Datasets}

For the Radish \cite{Radish} datasets, we selected the following three environments: Intel Research Lab, Orebro and MIT CSAIL. A single sequence was recorded for each environment, and for each sequence we divide into five disjoint sub-sequences with equal length, and label them into sub-sequence $1,2,3,4,5$. 


\begin{figure}[t]
    \centering
    \begin{subfigure}{0.2375 \textwidth}
        \centering
         \includegraphics[trim={7cm, 5cm, 6cm, 3.5cm}, clip, height=5cm]{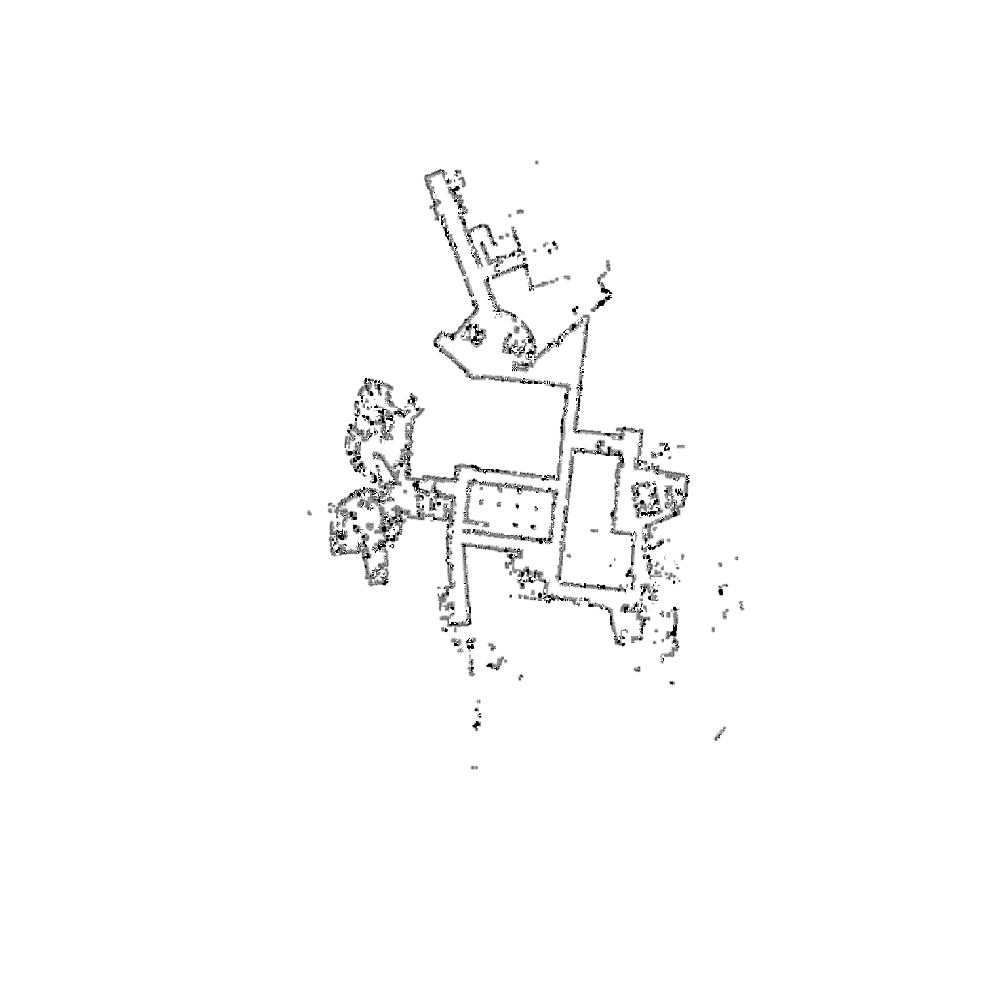}
    \subcaption{Central Agent}
    \end{subfigure}
    \begin{subfigure}{0.2375 \textwidth}
        \centering
         \includegraphics[trim={7cm, 5cm, 3cm, 3.5cm}, clip, height=5cm]{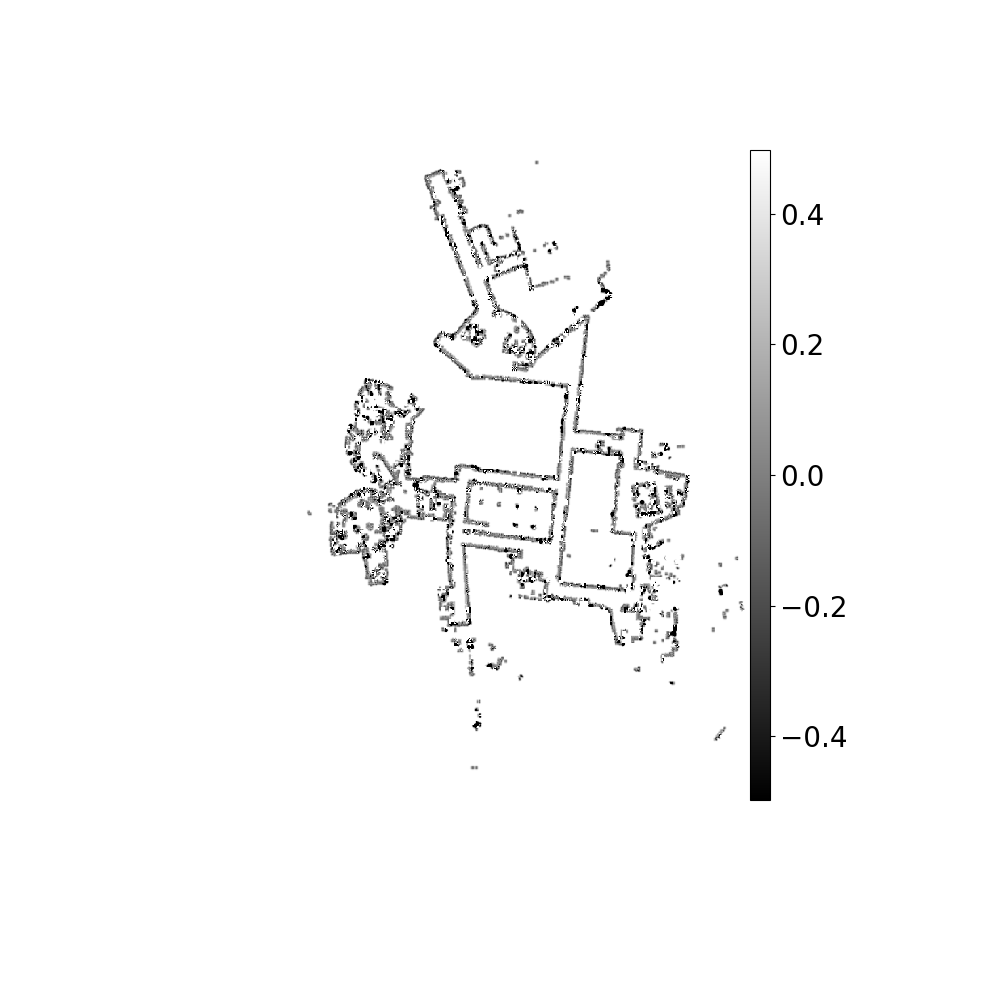}
    \subcaption{Agent 3}
    \end{subfigure}
    \caption{TSDF estimates on the MIT-CSAIL dataset from the Radish repository \cite{Radish}.}
    \label{fig:mit_csail}
\end{figure}

\begin{figure}[t]
    \centering
    \begin{subfigure}{0.235\textwidth}
         \includegraphics[trim={0.5cm, 2cm, 1cm, 2.75cm}, clip, width=\linewidth, height=3.5cm]{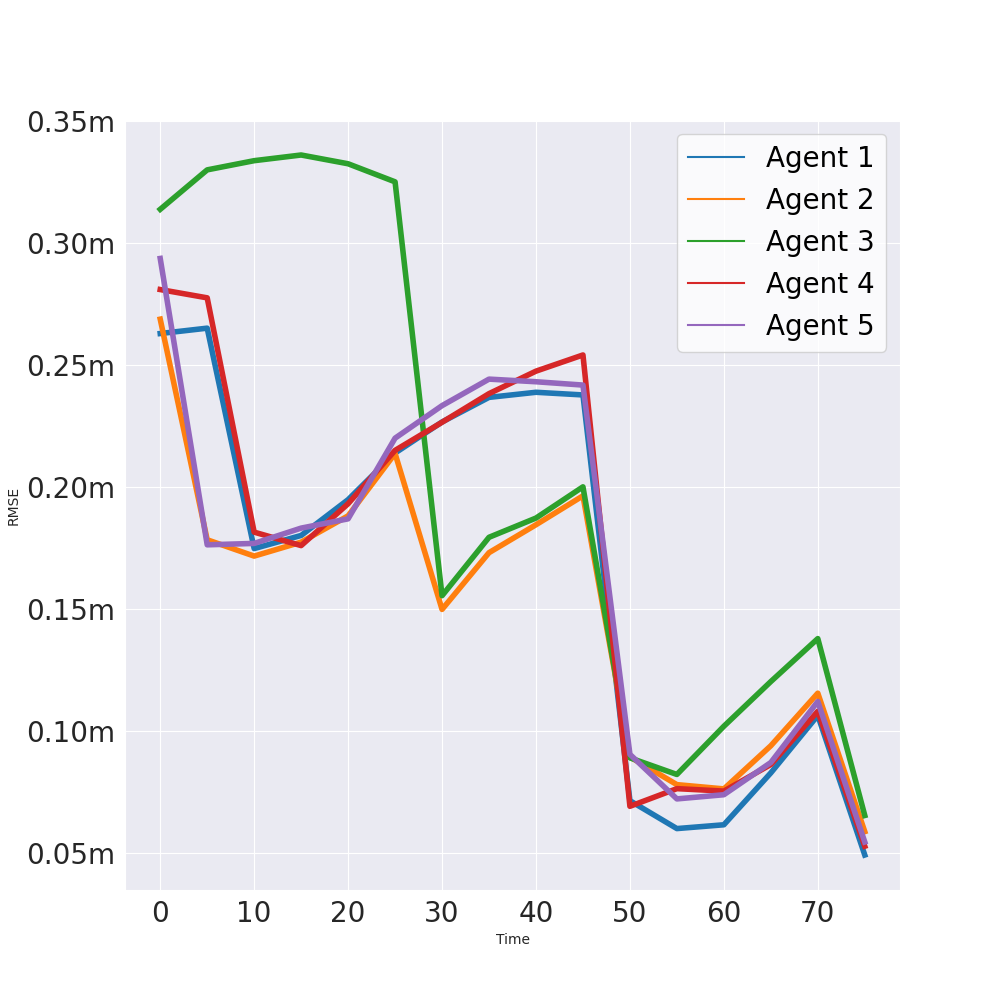}
    \subcaption{RMSE}
    \end{subfigure}
    \begin{subfigure}{0.235\textwidth}
         \includegraphics[trim={0.5cm, 2cm, 1.5cm, 2.75cm}, clip, width=\linewidth, height=3.5cm]{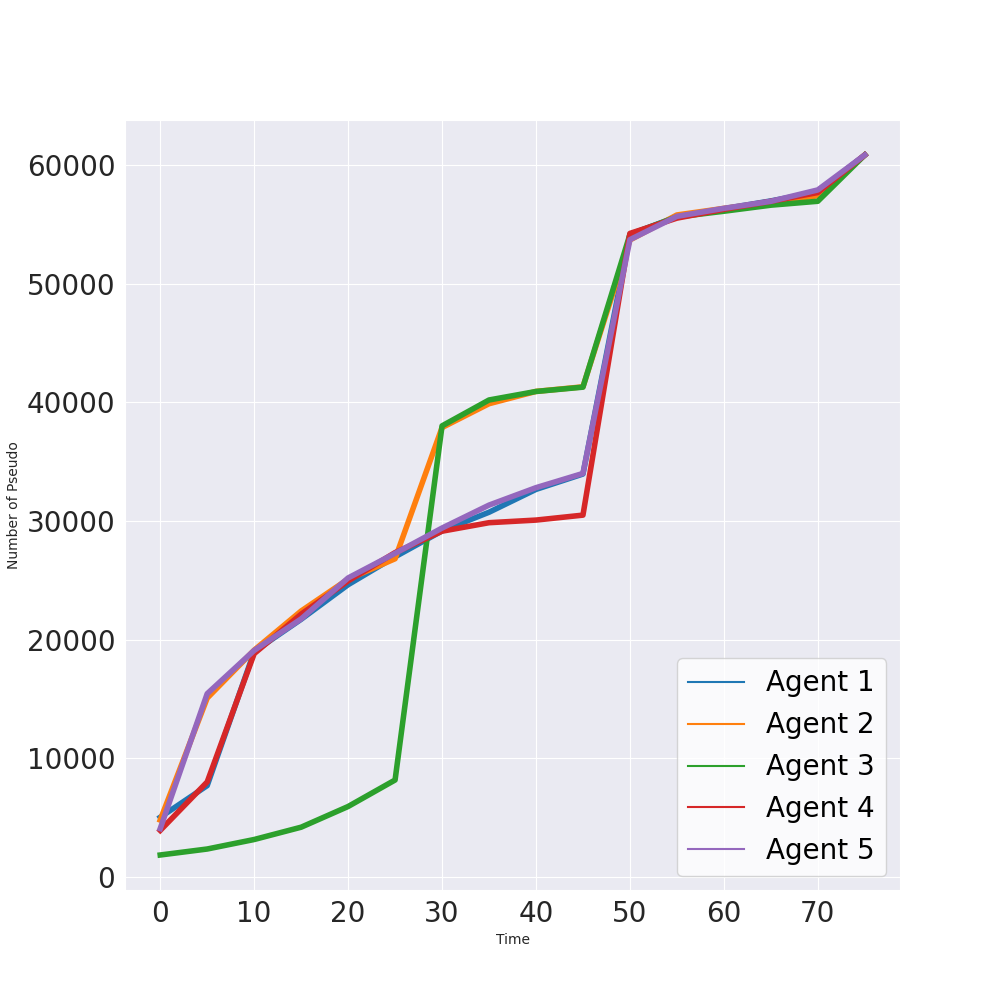}
    \subcaption{Number of Pseudo-Points}
    \end{subfigure}
\caption{Metrics for the MIT-CSAIL dataset from the Radish repository. As the robots explore new parts of the environment, the centralized agent accumulates more information than the distributed agents because some newly acquired messages are not passed to all agents. Hence the RMSE goes up for some periods of time. However, over the long-term horizon $T$ as the messages are passed to all agents the RMSE goes to zero asymptotically. }
\label{fig:metric_csail}
\end{figure}

\begin{figure*}[t]
    \begin{subfigure}{0.33 \textwidth}
         \includegraphics[trim={6cm, 3cm, 6cm, 3cm}, clip, width=\linewidth]{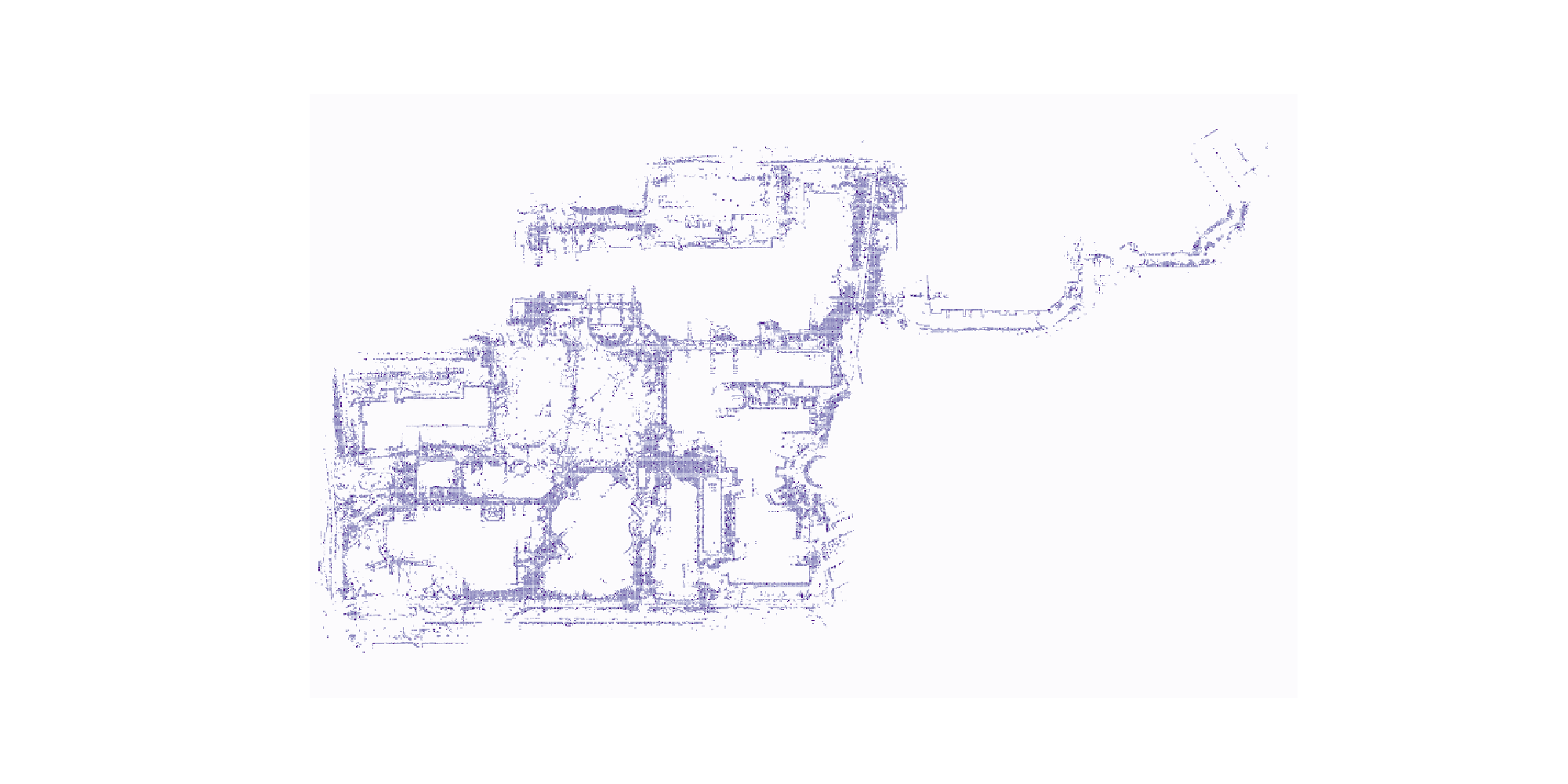}
    \subcaption{Central Agent}
    \end{subfigure}
    \begin{subfigure}{0.33 \textwidth}
         \includegraphics[trim={6cm, 3cm, 6cm, 3cm}, clip, width=\linewidth]{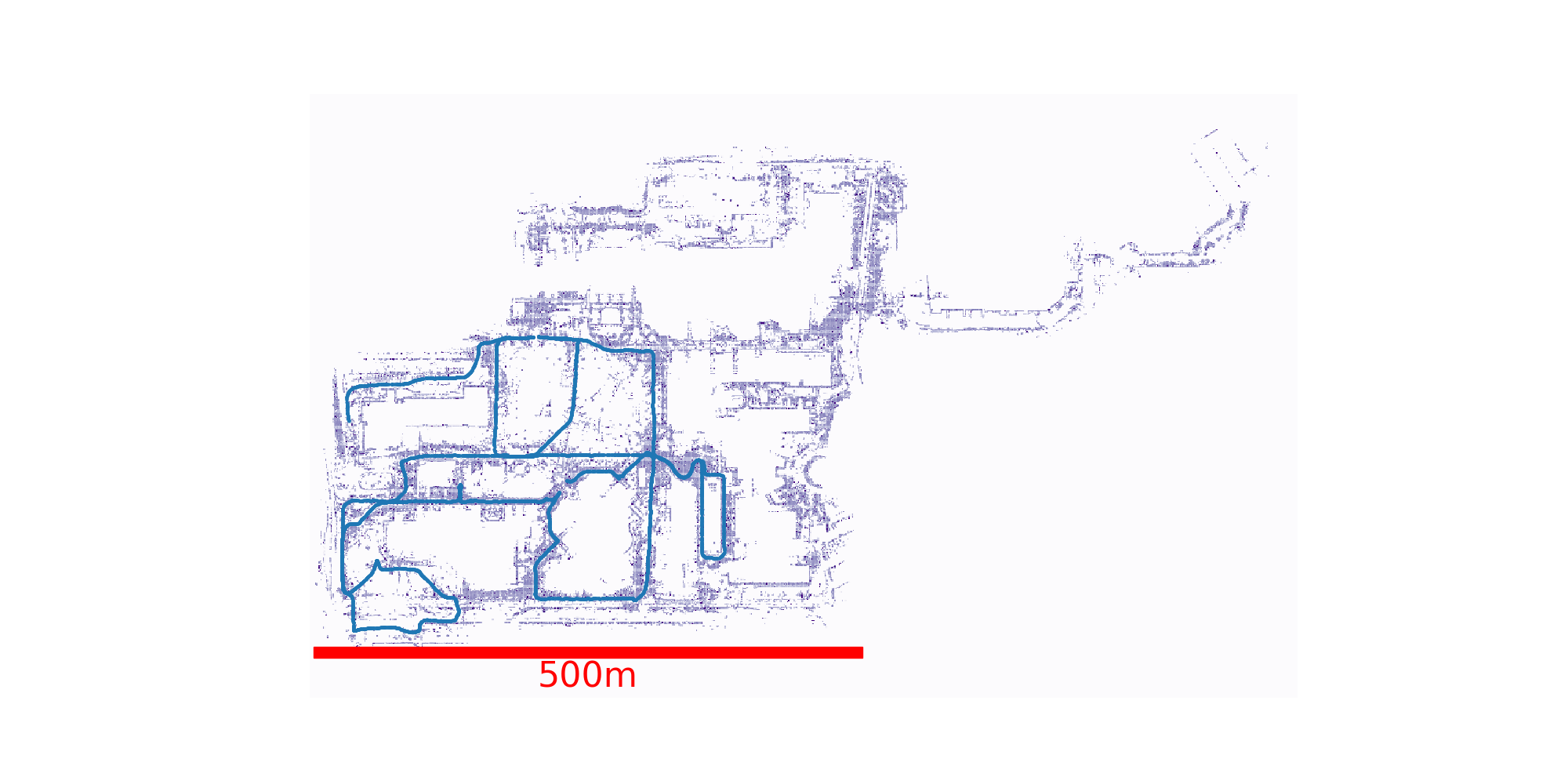}
    \subcaption{Agent 1}
    \end{subfigure}
    \begin{subfigure}{0.33 \textwidth}
         \includegraphics[trim={6cm, 3cm, 6cm, 3cm}, clip, width=\linewidth]{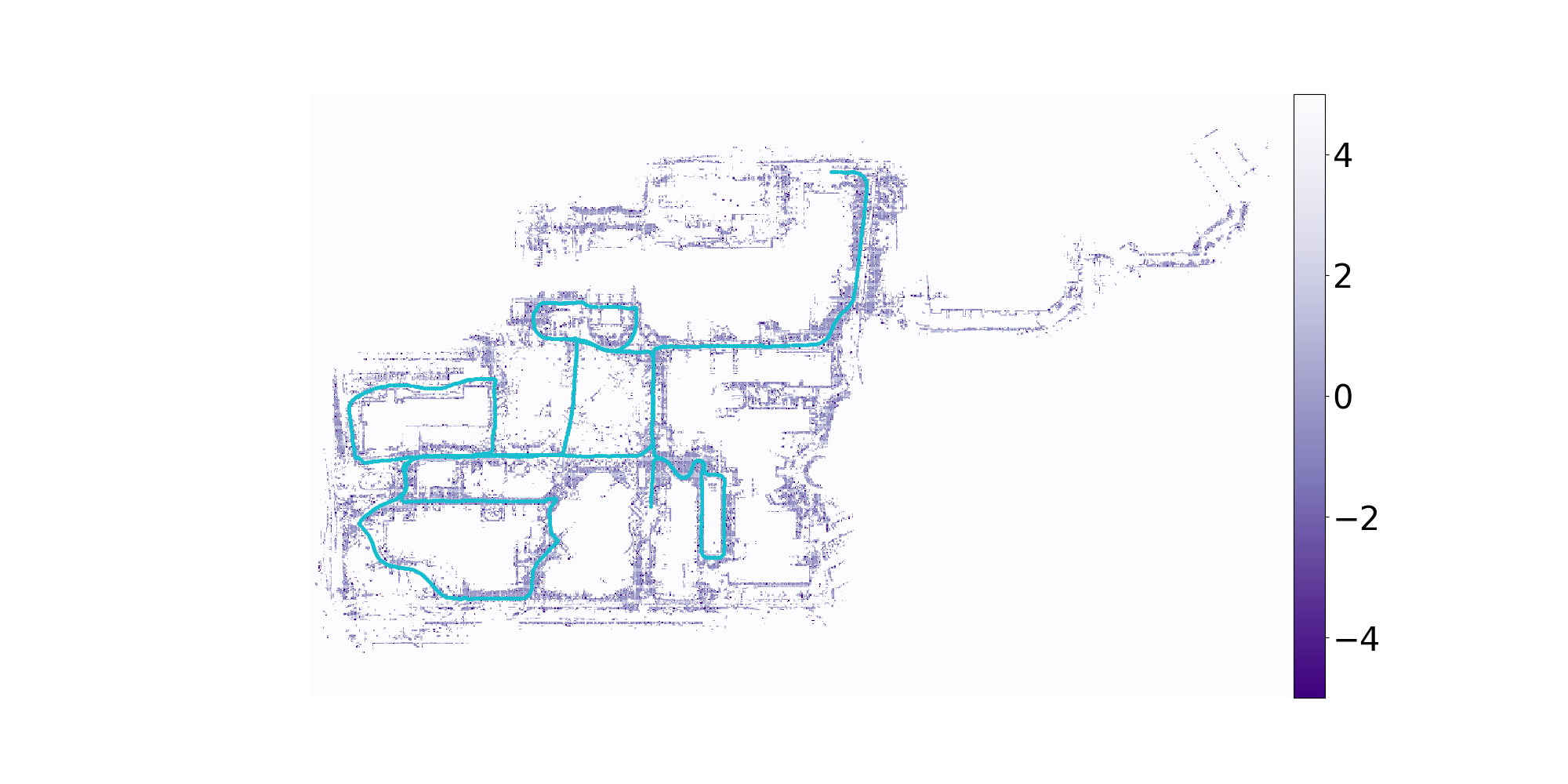}
    \subcaption{Agent 10}
    \end{subfigure}
\caption{TSDF estimates from the North-Campus Long-Term Dataset \cite{nclt}}
\label{fig:tsdf_nclt}
\end{figure*}

\begin{figure}[t]
    \centering
    \begin{subfigure}{0.235\textwidth}
         \includegraphics[trim={1.25cm, 2cm, 1cm, 3cm}, clip, width=\linewidth, height=3.5cm]{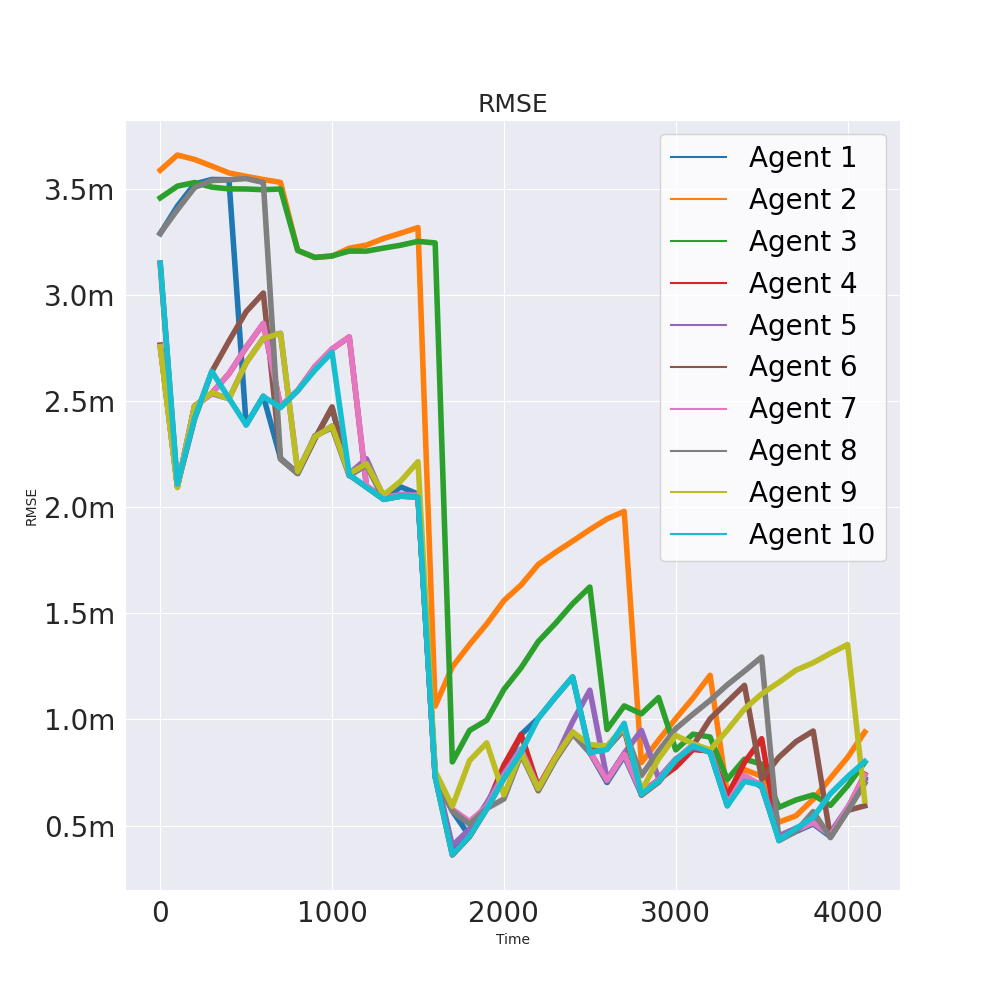}
    \subcaption{RMSE}
    \end{subfigure}
    \begin{subfigure}{0.235\textwidth}
         \includegraphics[trim={0.25cm, 2cm, 1.5cm, 3cm}, clip, width=\linewidth, height=3.5cm]{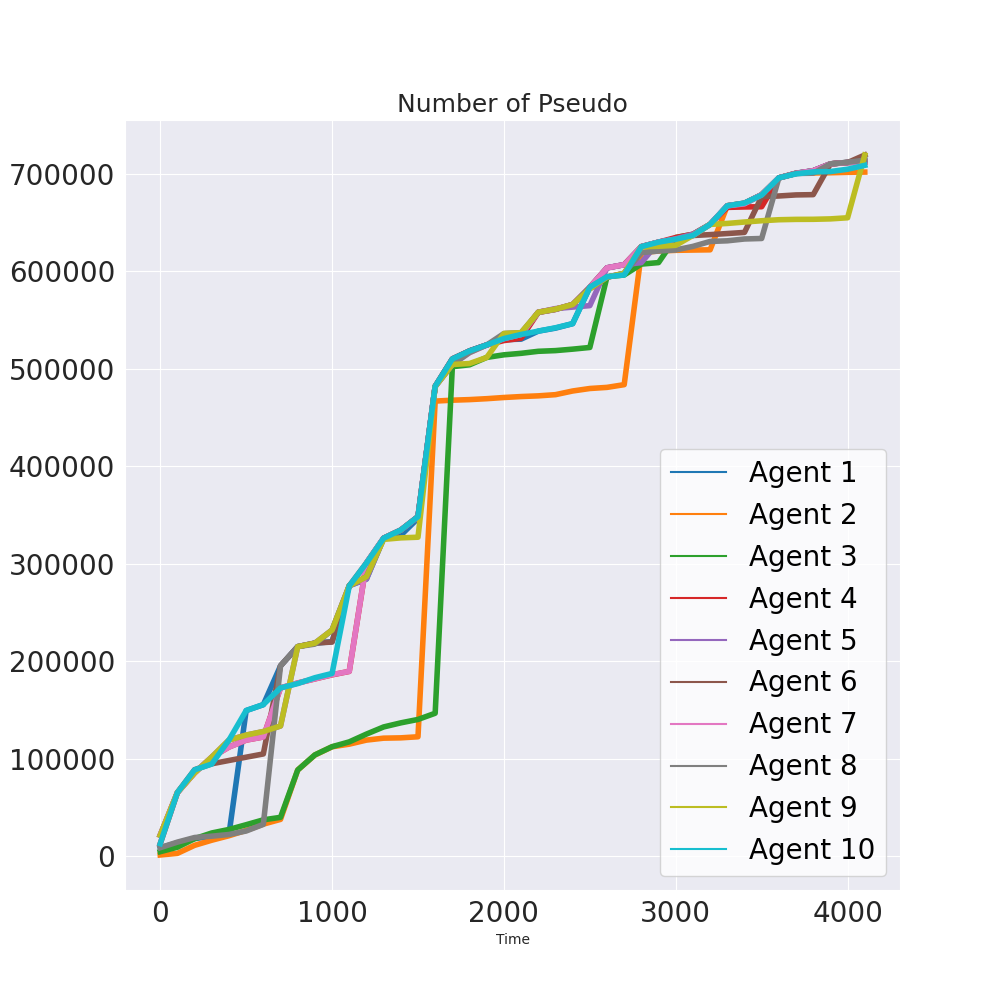}
    \subcaption{Number of Pseudo-Points}
    \end{subfigure}
\caption{Metrics over Time for the NCLT Dataset}
\label{metric_nclt}
\end{figure}

For each of the dataset, we set $h = 0.5$, and we assumed initially the entire environment to be in free space, i.e. $\mu_{0} = 0.5$. $r = 20m$ is selected as the communication range based on the size of the environments. We used a grid size of 0.1 for all agents in each dataset, and used the following parameters for the Gaussian Process: $c = 1.0$, $l = 0.1$, $\sigma=0.1$. In addition, we bound the max size of pseudo-points in each Quadtree leaf node to be 50 for efficient insertion and update of the pseudo-points. In Figure \ref{fig:mit_csail}, we show the \textit{TSDF} map of the central agent and agent 3 on MIT CSAIL and the estimated error with with respect to the centralized agent in Fig. \ref{fig:metric_csail}, and refer the readers to Appendix \ref{radish-results} for a complete list of our results of the agents, as well as on the other datasets.

\subsection{NCLT Dataset}

  The North Campus Long Term Dataset \cite{nclt} is a large-scale dataset both spatially and temporally. It covers roughly a square kilometer, includes challenging weather conditions and moving objects, and is collected over the time span of serveral months. The dataset provides 2D and 3D Lidar scans of the robot, camera images, odometry readings, as well as GPS positions of the robot. In our evaluation, we used the GPS readings as the robot's trajectory, and 2D Lidar scans as observations. We selected $10$ sequences from different dates where a different trajectory and observations are collected. 

For the Gaussian Process, the following GP parameters are used: $\sigma = 0.1, c=1.0, l=0.2$. A grid size of $0.25$ is used as the space covered is an order of magnitude larger than that in the Radish datasets. We use $5$ meters as the truncation value for TSDF, i.e. $\mu_{0} = h = 5$.

We show the TSDF estimates of the central agent, and agent 1 and 10 in Fig.\ref{fig:tsdf_nclt}. The RMSE and number of pseudo-points are shown in Fig.\ref{metric_nclt}, with communication range of $100m$. The results for the rest of the agents are similar, and are included in Appendix \ref{nclt-results}. When the robots are in isolated environments (e.g. robot 2 and 3 in the first 1500 timestamps), the error with respect to the central agent stagnated. However, when the robots rendez-vous and exchange information, the pseudo-point statistics are exchanged, and the number of pseudo-points increase sharply, while the error drops sharply at these moments of meetings. During some parts of the run the RMSE increases because the robot team continues to explore the environment, and may not have neighbors to exchange information. However, over the time-span $T$ the TSDF estimate of each individual agent continues getting closer to that of the central agent asymptotically.


In Table \ref{table:dist_thresh} we show the statistics corresponding to different values of $r$.  Larger value of $r$ corresponds to the robots have larger communication range, and hence higher likelihood a larger subset of agents to exchange information with at each time $t$. As $r$ increases, the mean and std of the RMSE both decrease at the end of the run. The number of pseudo-points and number of leaves both increase, as each robot has received more pseudo-points from a larger number of robots. 

\begin{table}[t]
    \centering
    \caption{Table over different values of communication range. Each entry is the mean and standard deviation computed over the 10 agents, at the end of the run. RMSE is computed according to \eqref{eq:rmse}.}
    \label{table:dist_thresh}
    \begin{tabular}{|c|c|c|c|}
    \hline
     Range (m)& RMSE(m) & \# Pseudo-points & \# Leaves\\
    \hline
    50 & $1.65 \pm 0.17$ & $627684 \pm 26247$ & $31372 \pm 1036$ \\ 
    \hline
    100 & $0.74 \pm 0.09$ & $712459 \pm 4777$ &  $ 34837 \pm 166$  \\
    \hline
    200 & $0.10 \pm 0.02$ & $719725 \pm 22$ & $35056 \pm 3.56$\\
    \hline
\end{tabular}
\end{table}

\section{Conclusion}
In summary, in this work we proposed a distributed, probabilistic, online and efficient algorithm for TSDF mapping using
robot team over time-varying communication graph. Then we provided theoretical guarantee for its consensus. We practically evaluated our method with large scale experiments on real world data sets. The directions that could be taken in future works includes Incorporating robot pose-estimation algorithm into the algorithm, and achieve a fully distributed $\textit{SLAM}$ system in real-time, and Utilizing the covariance of the GP estimates for down-stream tasks such as collision avoidance and motion planning.

\bibliographystyle{ieeetr}
\bibliography{root}

\clearpage
\onecolumn
\appendices
\section{Results for Radish Datasets}
\label{radish-results}
\begin{figure}[h!]
     \centering
     \begin{subfigure}{1.0\textwidth}
        \centering
         \includegraphics[trim={5cm, 2cm, 5cm, 3cm}, clip, width=0.89\linewidth]{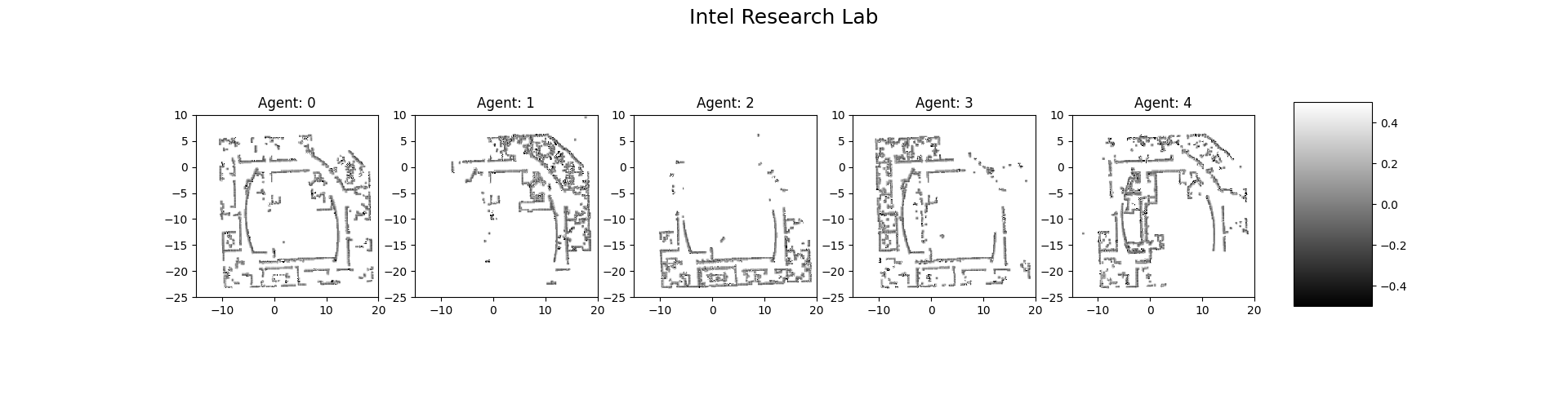}
         \includegraphics[trim={5cm, 2.5cm, 5cm, 3cm}, clip, width=0.89\linewidth]{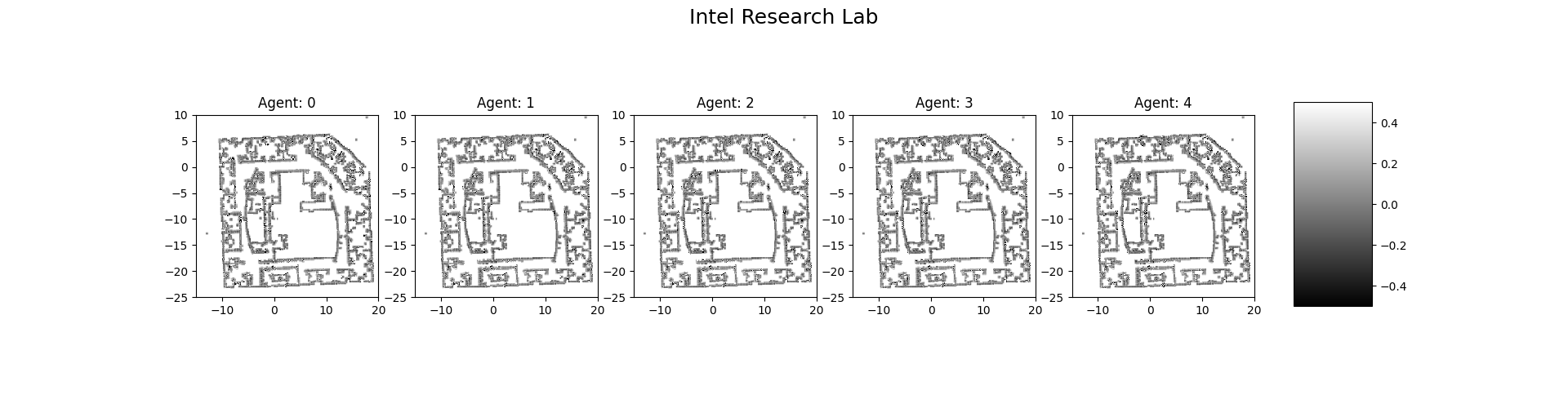}
     \caption{Intel Research Lab}
     \end{subfigure}
     
     \vspace{5pt}
     
     \begin{subfigure}{1.0\textwidth}
         \centering
         \includegraphics[trim={5cm, 2cm, 5cm, 2cm}, clip, width=0.89\linewidth]{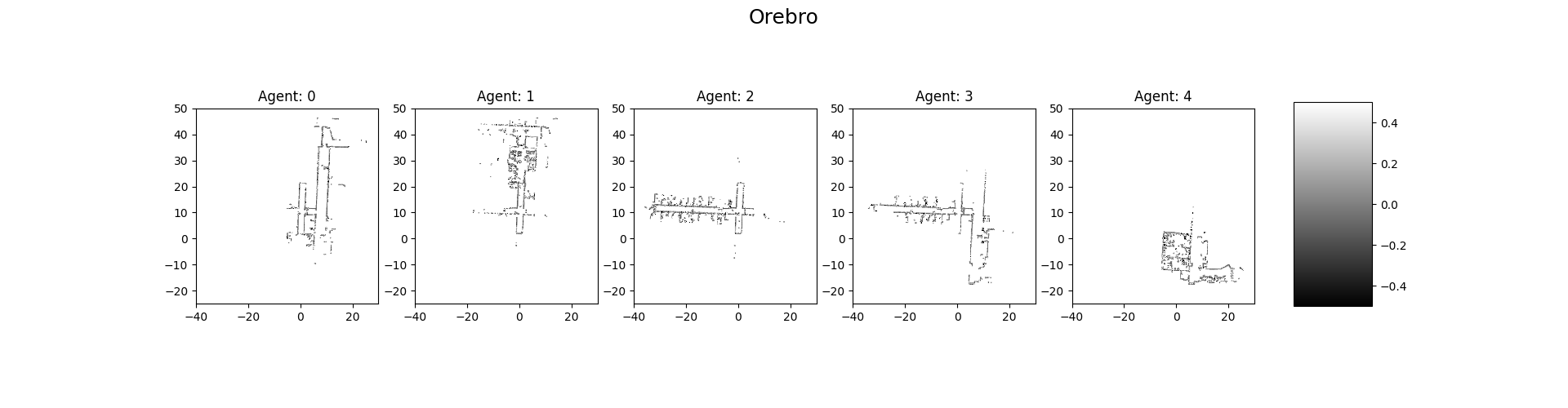}
         \includegraphics[trim={5cm, 2.5cm, 5cm, 2.5cm}, clip, width=0.89\linewidth]{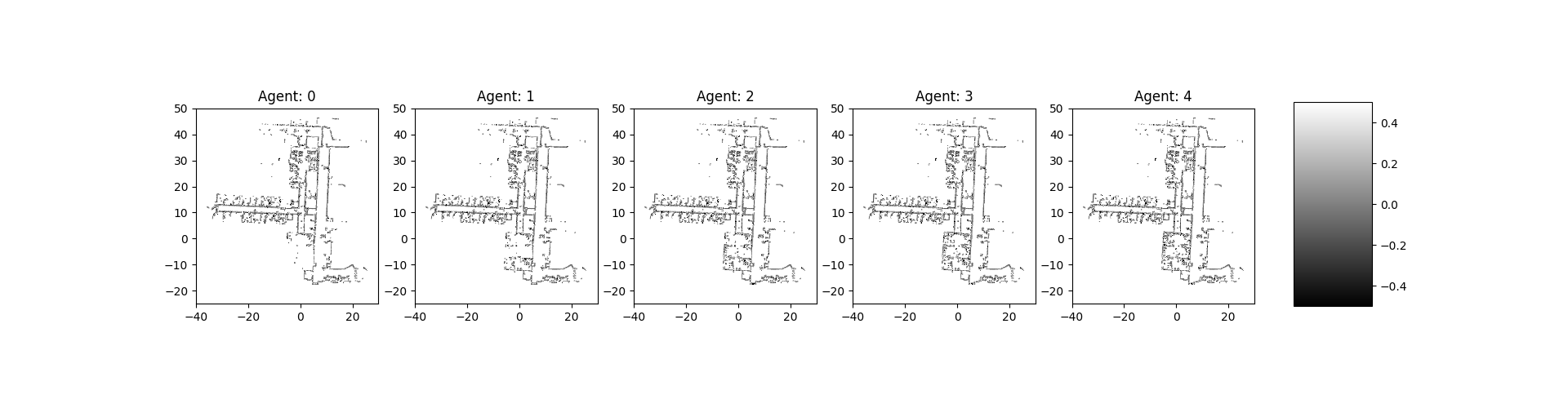}
     \caption{Orebro}
     \end{subfigure}
     
     \begin{subfigure}{1.0\textwidth}
          \centering
          \includegraphics[trim={5cm, 2cm, 5cm, 1cm}, clip, width=0.89\linewidth]{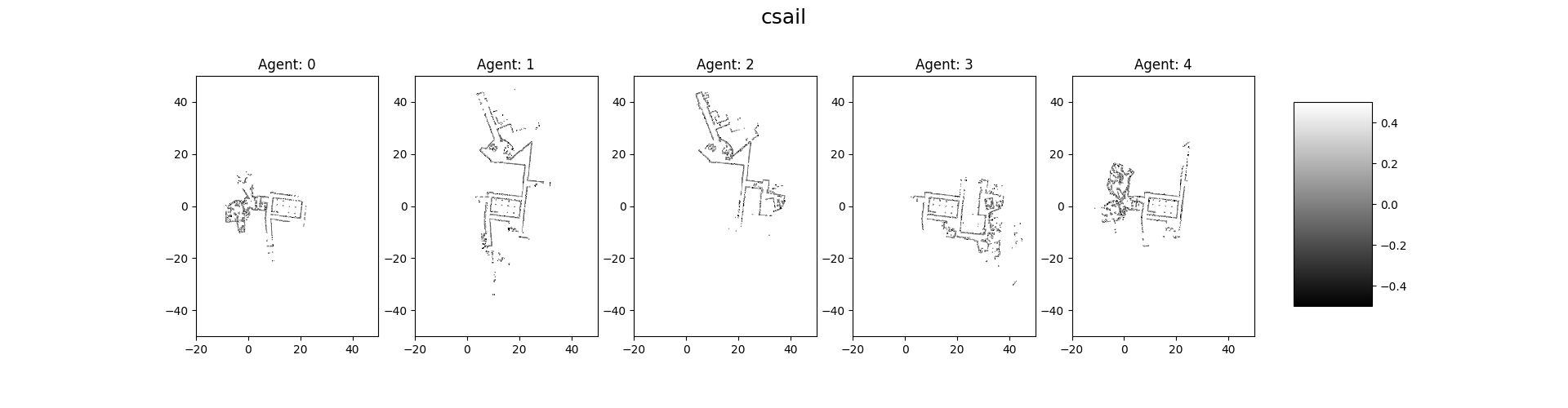}
         \includegraphics[trim={5cm, 1.5cm, 5cm, 1cm}, clip, width=0.89\linewidth]{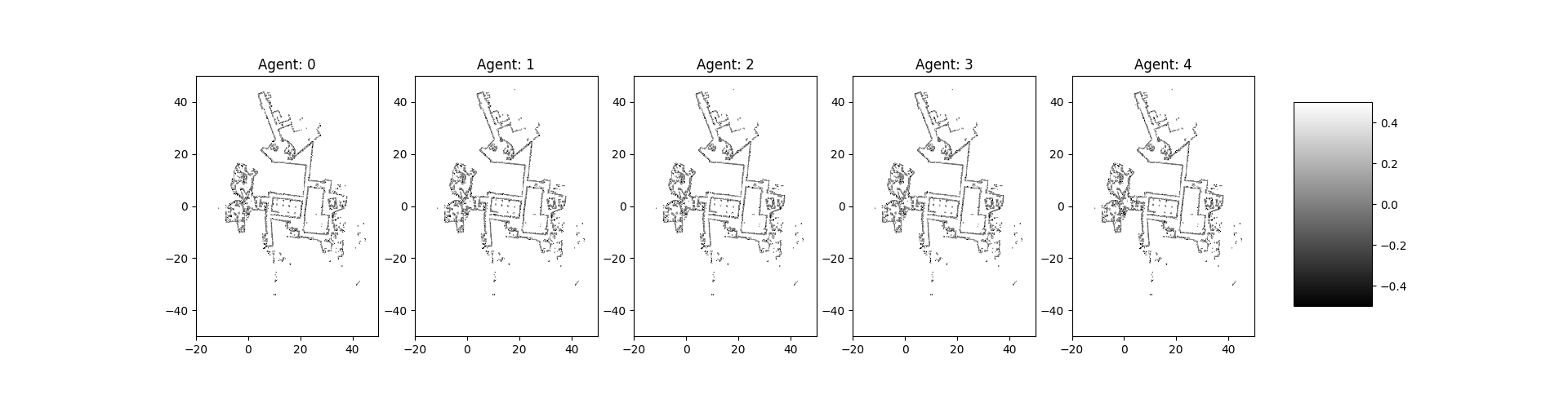}
     \caption{MIT CSAIL}
     \end{subfigure}
\end{figure}

\clearpage
\section{Results for North Campus Long-Term Dataset}
\label{nclt-results}
\begin{figure}[h!]
    \begin{subfigure}{0.33 \textwidth}
         \includegraphics[trim={6cm, 3cm, 6cm, 3cm}, clip, width=\linewidth]{Figures/NCLT/tsdf/agent1_tsdf.png}
    \subcaption{Agent 1}
    \end{subfigure}
    \begin{subfigure}{0.33 \textwidth}
         \includegraphics[trim={6cm, 3cm, 6cm, 3cm}, clip, width=\linewidth]{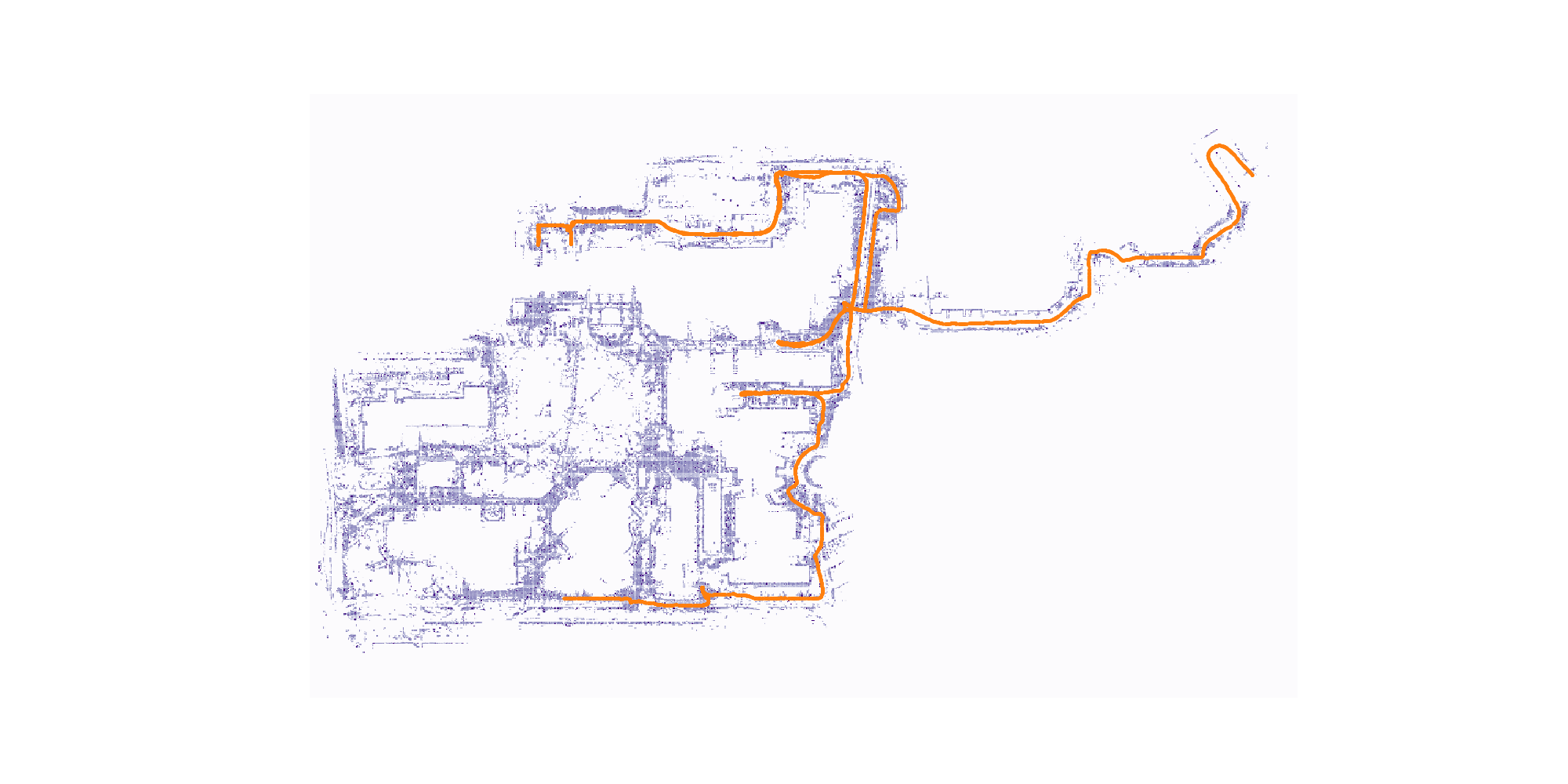}
    \subcaption{Agent 2}
    \end{subfigure}
    \begin{subfigure}{0.33 \textwidth}
         \includegraphics[trim={6cm, 3cm, 6cm, 3cm}, clip, width=\linewidth]{Figures/NCLT/tsdf/agent3_tsdf.png}
    \subcaption{Agent 3}
    \end{subfigure}
    \begin{subfigure}{0.33 \textwidth}
         \includegraphics[trim={6cm, 3cm, 6cm, 3cm}, clip, width=\linewidth]{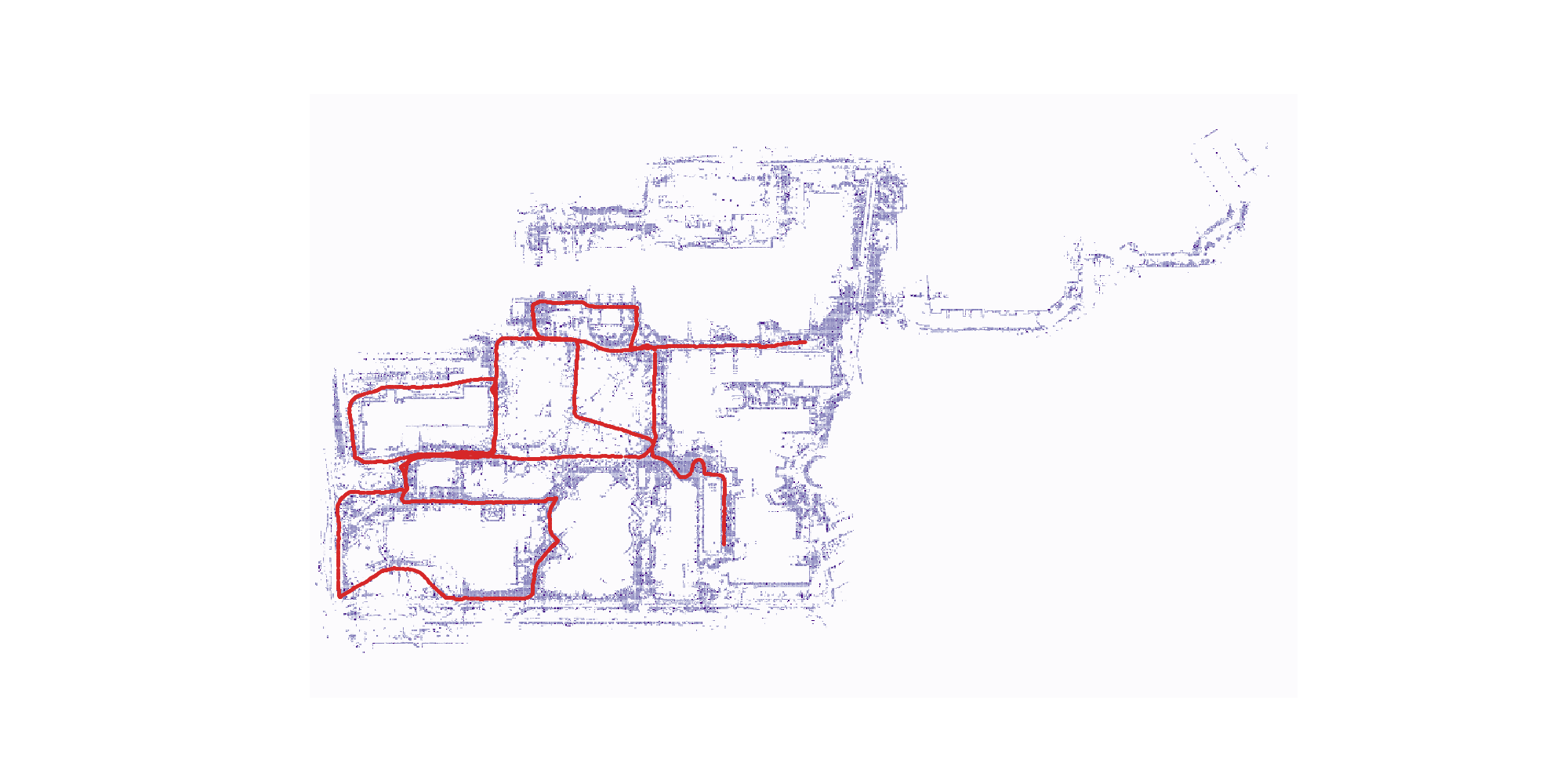}
    \subcaption{Agent 4}
    \end{subfigure}
    \begin{subfigure}{0.33 \textwidth}
         \includegraphics[trim={6cm, 3cm, 6cm, 3cm}, clip, width=\linewidth]{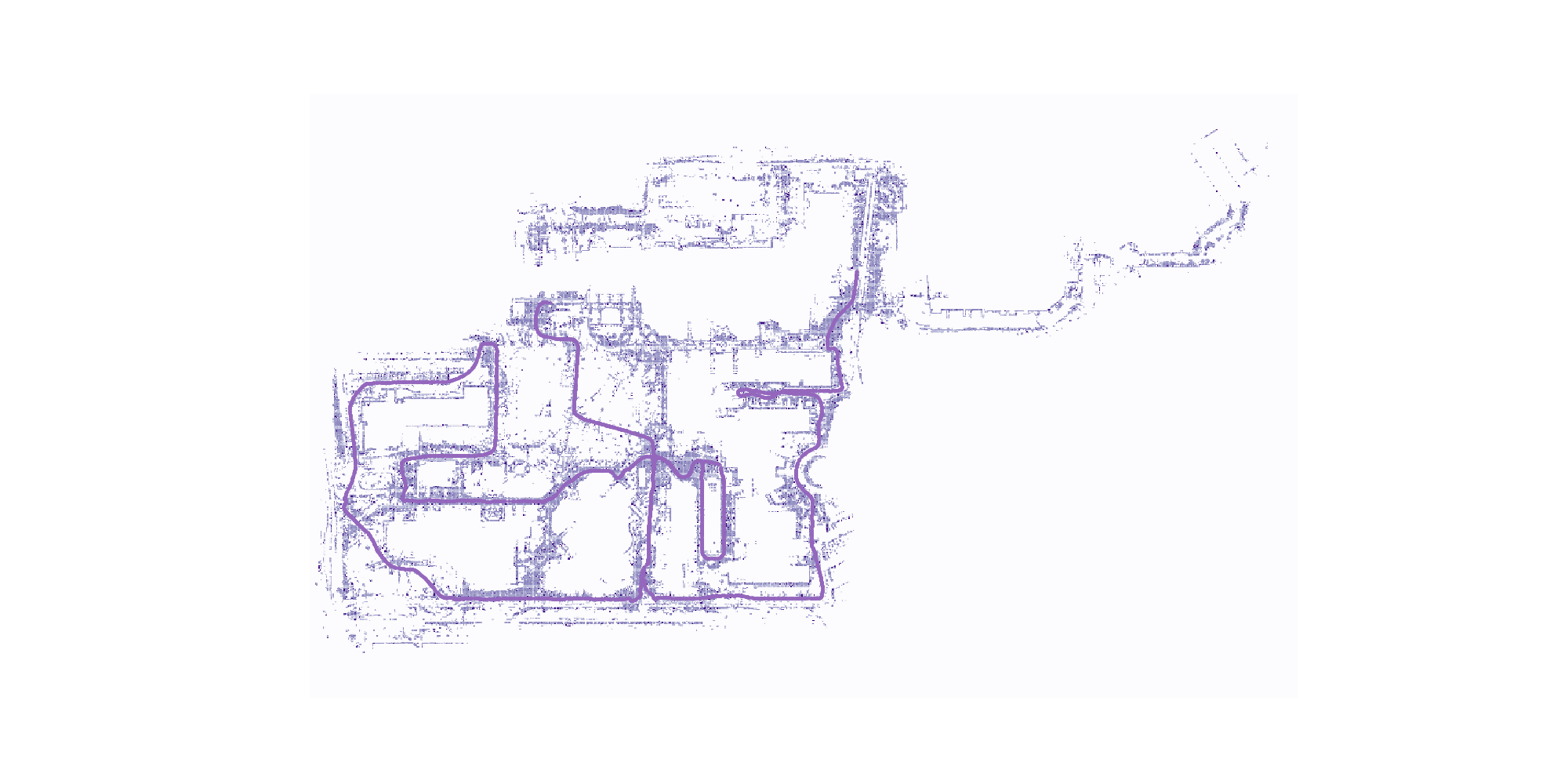}
    \subcaption{Agent 5}
    \end{subfigure}
    \begin{subfigure}{0.33 \textwidth}
         \includegraphics[trim={6cm, 3cm, 6cm, 3cm}, clip, width=\linewidth]{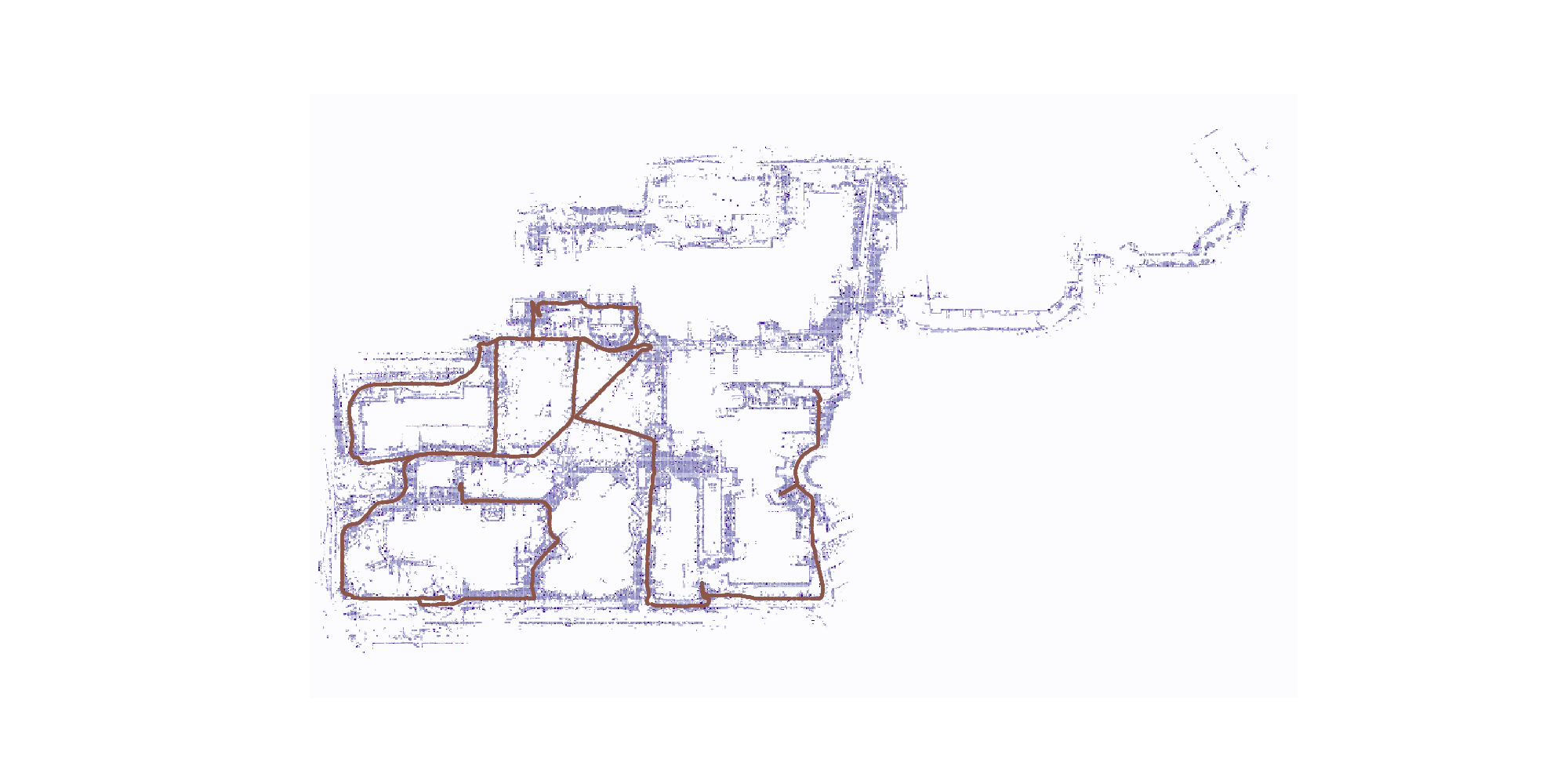}
    \subcaption{Agent 6}
    \end{subfigure}
    \begin{subfigure}{0.33 \textwidth}
         \includegraphics[trim={6cm, 3cm, 6cm, 3cm}, clip, width=\linewidth]{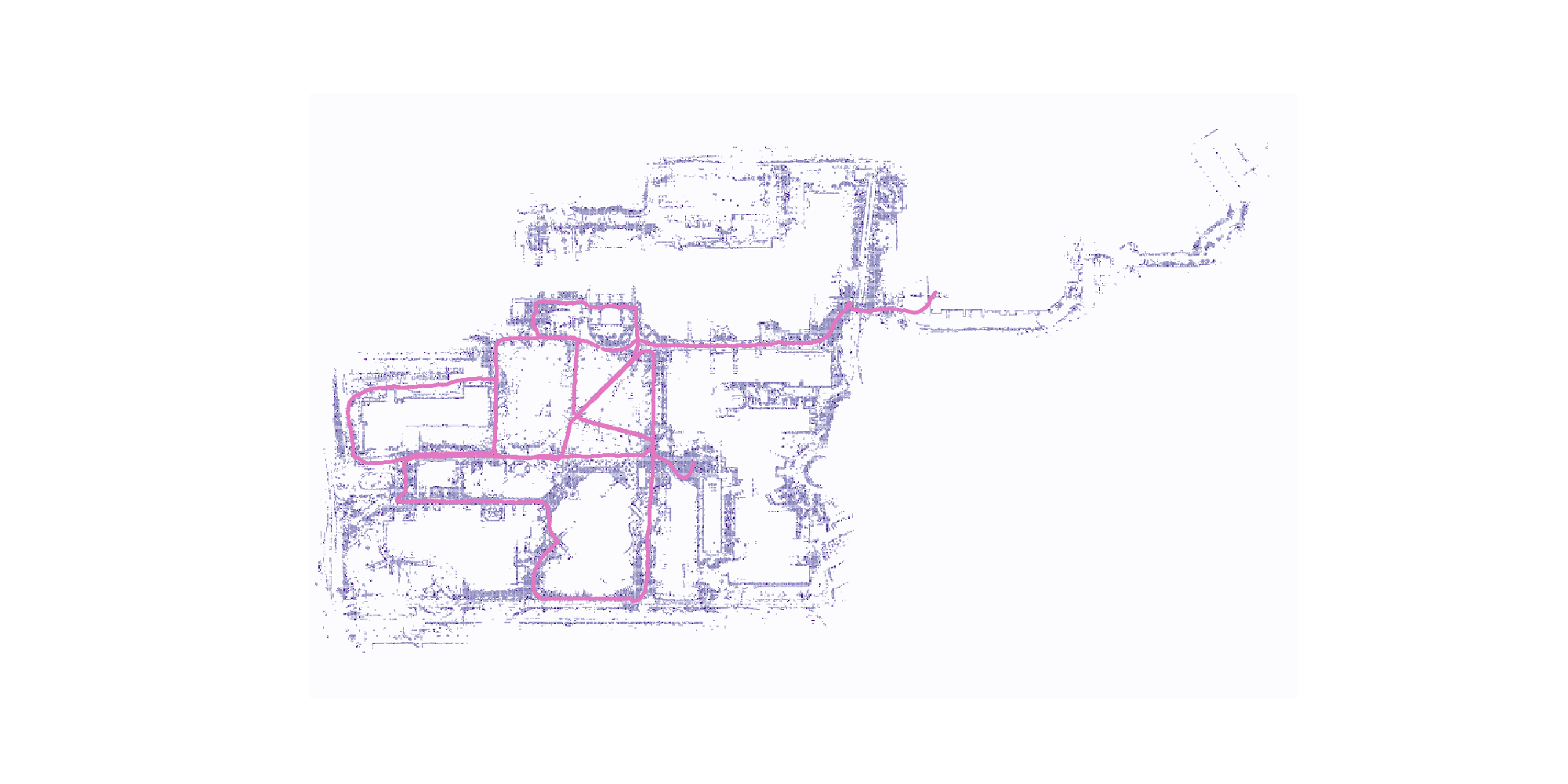}
    \subcaption{Agent 7}
    \end{subfigure}
    \begin{subfigure}{0.33 \textwidth}
         \includegraphics[trim={6cm, 3cm, 6cm, 3cm}, clip, width=\linewidth]{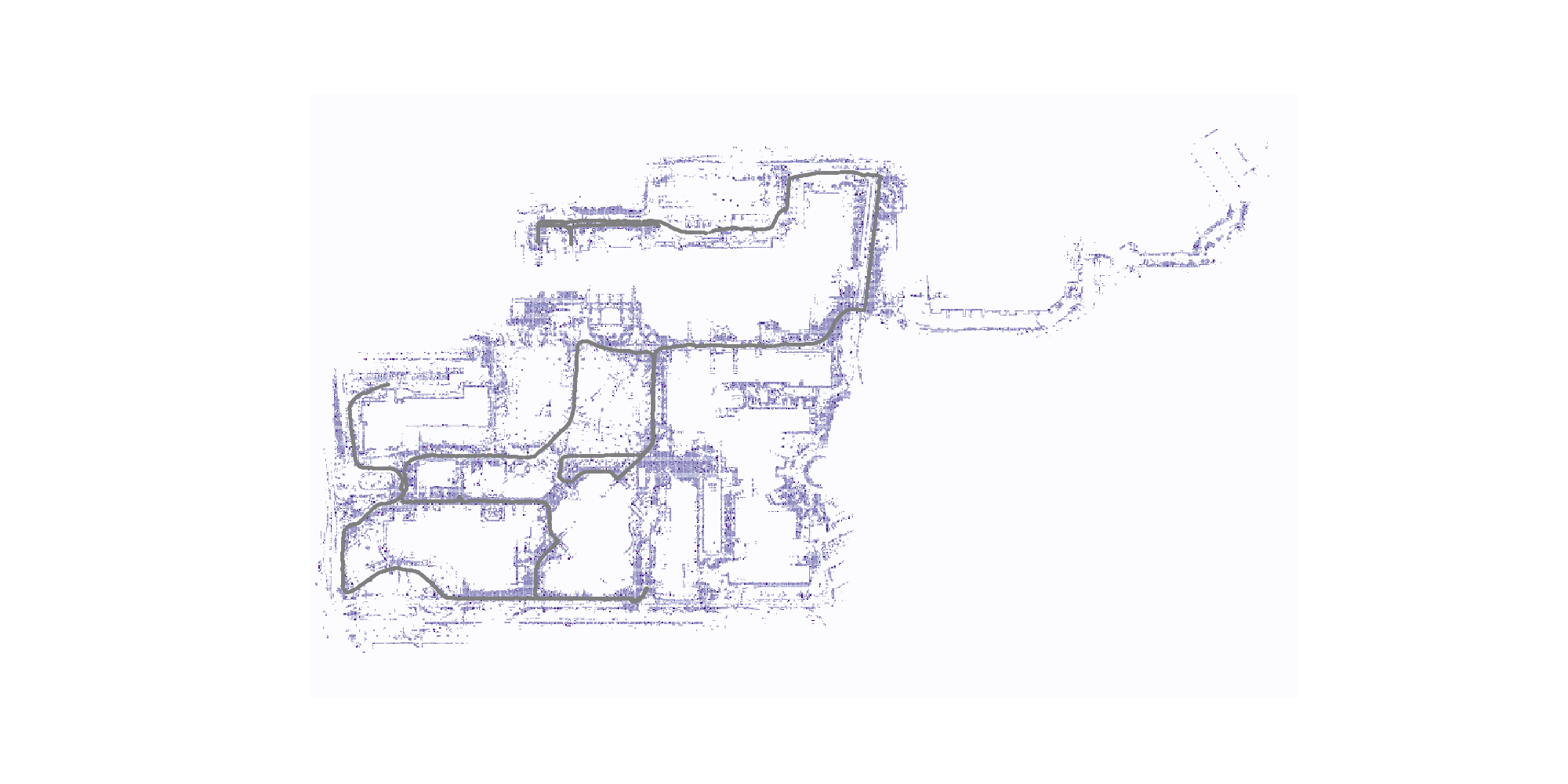}
    \subcaption{Agent 8}
    \end{subfigure}
    \begin{subfigure}{0.33 \textwidth}
         \includegraphics[trim={6cm, 3cm, 6cm, 3cm}, clip, width=\linewidth]{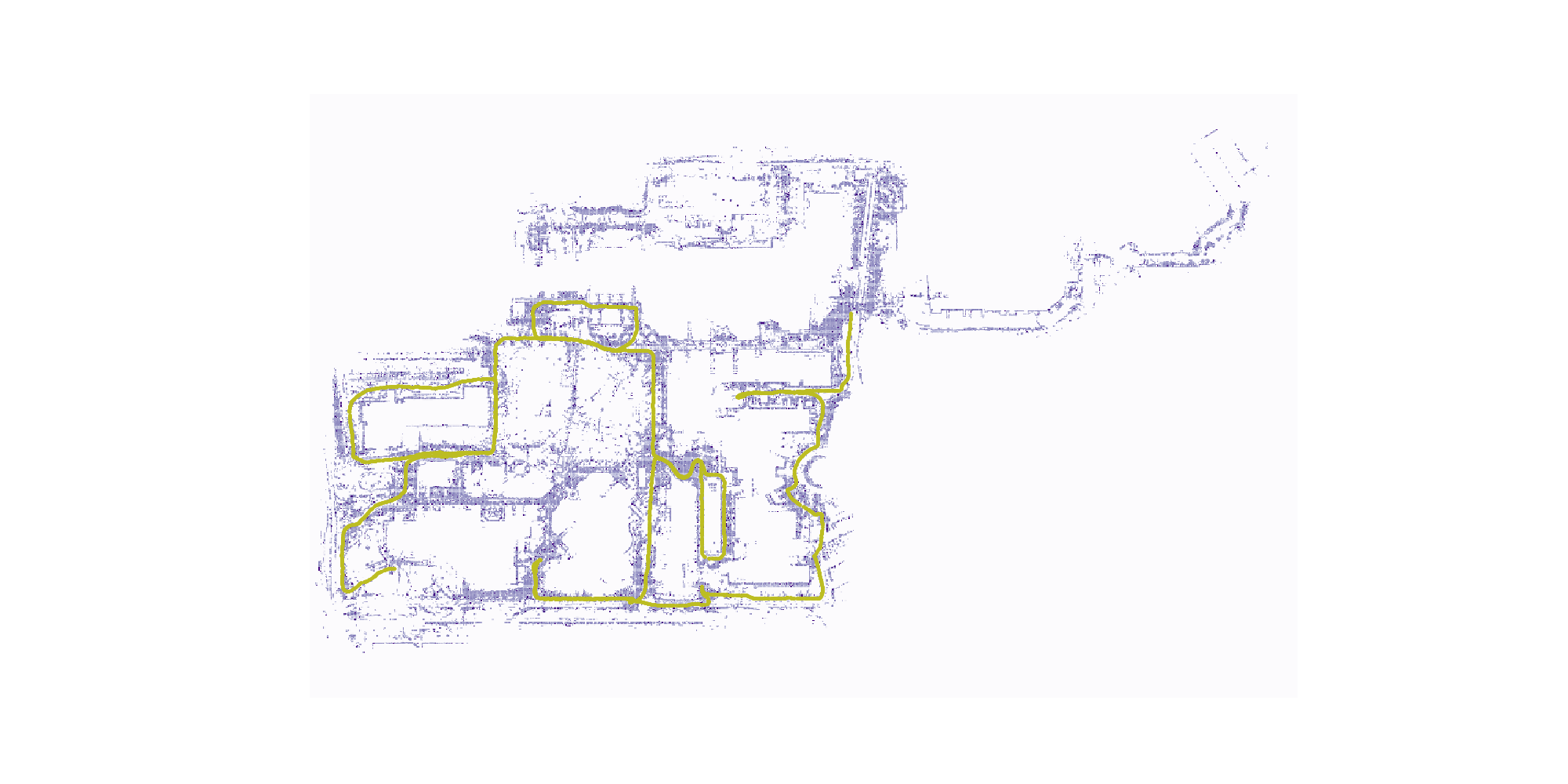}
    \subcaption{Agent 9}
    \end{subfigure}
    \begin{subfigure}{0.33 \textwidth}
         \includegraphics[trim={6cm, 3cm, 6cm, 3cm}, clip, width=\linewidth]{Figures/NCLT/tsdf/agent10_tsdf.png}
    \subcaption{Agent 10}
    \end{subfigure}
\end{figure}





\end{document}